\documentclass[10pt,twocolumn,letterpaper]{article}

\usepackage[pagenumbers]{cvpr} 

\usepackage{graphicx}
\usepackage{amsmath}
\usepackage{amssymb}
\usepackage{booktabs}

\newcommand{\wangcan}[1]{{\color{black} #1}}

\usepackage[pagebackref=true,breaklinks=true,colorlinks=true,bookmarks=false,linkcolor={red!50!black},urlcolor={darkgray!80!black},citecolor={green!50!black}]{hyperref} 

\usepackage[capitalize]{cleveref}
\crefname{section}{Sec.}{Secs.}
\Crefname{section}{Section}{Sections}
\Crefname{table}{Table}{Tables}
\crefname{table}{Tab.}{Tabs.}

\usepackage{xcolor}
\usepackage{enumitem}
\usepackage{xspace}
\usepackage{booktabs}
\usepackage{colortbl}
\usepackage{multirow}
\usepackage{makecell}
\usepackage{lipsum} 
\usepackage{gensymb} 
\usepackage{algorithm}
\usepackage{algpseudocode}

\newtheorem{theorem}{Theorem}[section]

\newtheorem{proposition}[theorem]{Proposition}
\newenvironment{proof}{\paragraph{Proof:}}{\hfill$\square$}
\usepackage{titletoc}

\usepackage[labelsep=period]{caption}
\renewcommand{\paragraph}[1]{\vspace{0.2em}\noindent \textbf{#1 \hspace{0.2em}}}

\usepackage{amsmath,amsfonts,bm}

\def\eqref#1{equation~\ref{#1}}

\def\1{\bm{1}}

\def\rmI{{\mathbf{I}}}

\def\vp{{\bm{p}}}

\def\vu{{\bm{u}}}
\def\vv{{\bm{v}}}
\def\vw{{\bm{w}}}
\def\vx{{\bm{x}}}
\def\vy{{\bm{y}}}

\DeclareMathAlphabet{\mathsfit}{\encodingdefault}{\sfdefault}{m}{sl}
\SetMathAlphabet{\mathsfit}{bold}{\encodingdefault}{\sfdefault}{bx}{n}

\newcommand{\x}{{\boldsymbol x}}

\usepackage{capt-of}
\usepackage{diagbox}

\definecolor{C0}{rgb}{0.121569, 0.466667, 0.705882}
\definecolor{C1}{rgb}{1.000000, 0.498039, 0.054902}
\definecolor{C2}{rgb}{0.172549, 0.627451, 0.172549}
\definecolor{C3}{rgb}{0.839216, 0.152941, 0.156863}
\definecolor{C4}{rgb}{0.580392, 0.403922, 0.741176}
\definecolor{C5}{rgb}{0.549020, 0.337255, 0.294118}
\definecolor{C6}{rgb}{0.890196, 0.466667, 0.760784}
\definecolor{C7}{rgb}{0.498039, 0.498039, 0.498039}
\definecolor{C8}{rgb}{0.737255, 0.741176, 0.133333}
\definecolor{C9}{rgb}{0.090196, 0.745098, 0.811765}
\definecolor{trolleygrey}{rgb}{0.5, 0.5, 0.5}

\definecolor{BrickRed}{rgb}{0.6,0,0}
\definecolor{RoyalBlue}{rgb}{0,0,0.8}
\definecolor{Tdgreen}{rgb}{0,0.4,0.7}
\definecolor{pinegreen}{rgb}{0.0, 0.47, 0.44}
\definecolor{cornellred}{rgb}{0.7, 0.11, 0.11}
\definecolor{cadmiumgreen}{rgb}{0.0, 0.42, 0.24}
\definecolor{spirodiscoball}{rgb}{0.06, 0.75, 0.99}
\definecolor{mylightblue}{rgb}{0.85, 0.90, 0.94}
\definecolor{maroon}{cmyk}{0,0.87,0.68,0.32}

\usepackage{pifont}
\definecolor{cfg}{rgb}{0.906, 0.435, 0.318}
\definecolor{cfgpp}{rgb}{0.165, 0.616, 0.561}
\definecolor{cfgnull}{rgb}{0.208, 0.565, 0.953}

\newcommand{\todo}[1]{{\color{black}#1}}

\definecolor{Gray}{gray}{0.9}

\definecolor{MyDarkRed}{rgb}{0.46, 0.16, 0.16}
\definecolor{MyDarkBlue}{rgb}{0.16, 0.16, 0.66}

\newcommand{\MethodName}{\textsc{Ani3DHuman}\xspace}

\def\paperID{} 
\def\confName{CVPR}
\def\confYear{2026}

\begin{document}

\title{\MethodName: Photorealistic 3D Human Animation with Self-guided Stochastic Sampling}

\author{
    Qi Sun$^{1}$ \quad 
    Can Wang$^{1}$ \quad 
    Jiaxiang Shang \quad 
    Yingchun Liu \quad 
    Jing Liao$^{1}$\thanks{Corresponding author} \\
    $^{1}$City University of Hong Kong
}

\newcommand\blfootnote[1]{%
  \begingroup
  \renewcommand\thefootnote{}\footnote{#1}%
  \addtocounter{footnote}{-1}%
  \endgroup
}

\twocolumn[{%
    \renewcommand\twocolumn[1][]{#1}%
    \maketitle
    \begin{center}
        \centering
        \includegraphics[width=\linewidth]{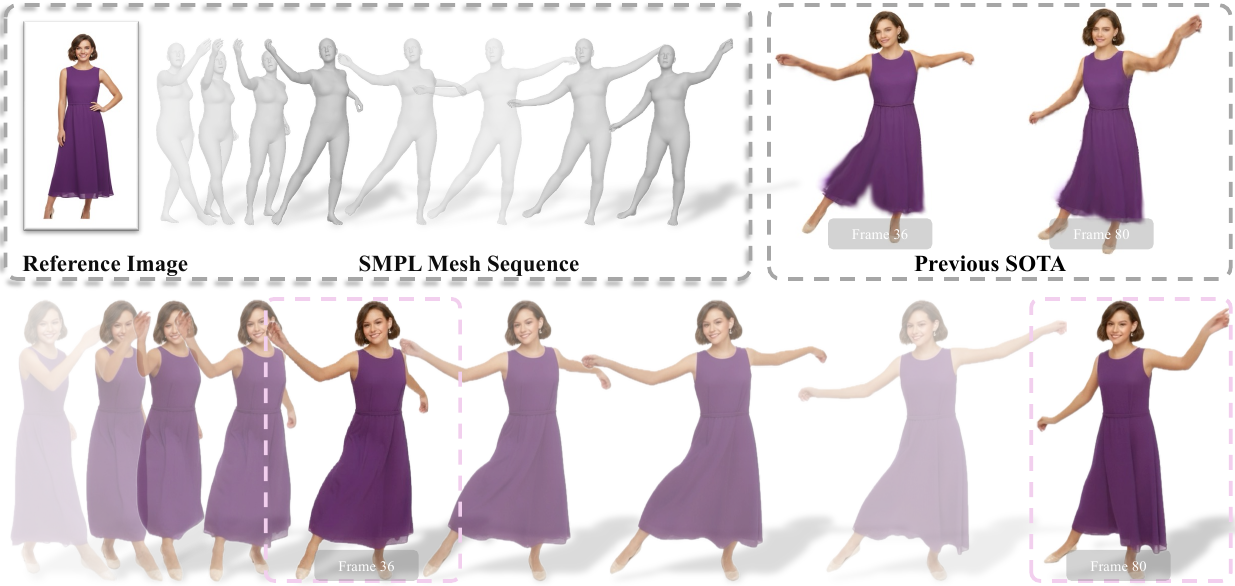}
        \captionof{figure}{
        Given a reference human image and a target SMPL mesh sequence, our method synthesizes photorealistic 3D human animation. Unlike the previous state-of-the-art (SOTA) methods (e.g., LHM~\cite{qiu2025LHM} \textbf{(top-right))} that are limited to rigid motion, our \MethodName \textbf{(bottom)} can further generate high-fidelity nonrigid dynamics, capturing the natural flow of the dress.
        }
        \label{fig:teaser}
    \end{center}
    \vspace{1.5em}
}]

\blfootnote{$^*$Corresponding author.}

\vspace{-4mm}
\begin{abstract}
  Current 3D human animation methods struggle to achieve photorealism: kinematics-based approaches lack non-rigid dynamics (e.g., clothing dynamics), while methods that leverage video diffusion priors can synthesize non-rigid motion but suffer from quality artifacts and identity loss. To overcome these limitations, we present \MethodName, a framework that marries kinematics-based animation with video diffusion priors. We first introduce a layered motion representation that disentangles rigid motion from residual non-rigid motion. Rigid motion is generated by a kinematic method, which then produces a coarse rendering to guide the video diffusion model in generating video sequences that restore the residual non-rigid motion. However, this restoration task, based on diffusion sampling, is highly challenging, as the initial renderings are out-of-distribution, causing standard deterministic ODE samplers to fail. Therefore, we propose a novel self-guided stochastic sampling method, which effectively addresses the out-of-distribution problem by combining stochastic sampling (for photorealistic quality) with self-guidance (for identity fidelity). These restored videos provide high-quality supervision, enabling the optimization of the residual non-rigid motion field. Extensive experiments demonstrate that \MethodName can generate photorealistic 3D human animation, outperforming existing methods. Code is available in \href{https://github.com/qiisun/ani3dhuman}{https://github.com/qiisun/ani3dhuman}.
\end{abstract}

\section{Introduction}
\label{sec:intro}

The importance of 3D digital humans has been growing across various applications, including AR/VR~\cite{genay_tvcg}, gaming~\cite{cai2024playing}, education~\cite{lakhfif2020design}, and healthcare~\cite{korban2022survey}. 
This has motivated numerous research efforts aimed at automatically animating 3D humans~\cite{qiu2025LHM, sim2025persona, moon2024exavatar}.

In traditional 3D human animation, researchers drive the motion with either kinematics-based methods, such as skeletons~\cite{sim2025persona, moon2024exavatar} and SMPL meshes~\cite{qiu2025LHM}, or physics-based methods~\cite{PhysAavatar24, grigorev2023hood}. 
Kinematics-based methods offer a controllable way to describe rigid human movement, but are challenging to model non-rigid deformations, such as soft body movements, clothing, and hair, which involve complex, flexible changes in shape and structure.  
Physics-based methods~\cite{PhysAavatar24, grigorev2023hood, xiang2022dressing} focus on modeling the complex dynamic effects of clothing interacting with human bodies. 
While effective at generating natural and realistic non-rigid garment animations, these methods require high computational resources and involve significant complexity in specifying physical models and numerous physical parameters.

Recent advancements in video diffusion models~\cite{videoworldsimulators2024, wan2025} offer a compelling alternative, inherently modeling both rigid and non-rigid dynamics without physical simulation. 
\todo{Using score distillation sampling~\cite{poole2022dreamfusion, singer2023text, bah20244dfy} to distill motion suffers from unsatisfactory results like over-saturation. 
Another approach first uses a diffusion model to generate videos, and then directly reconstructs 3D animation from them.
However, this pipeline suffers from distinct failure modes originating from video generation:
1) When relying on multi-view diffusion models~\cite{xie2024sv4d, liang2024diffusion4d, gao2025charactershot}, the reconstruction quality is limited by the scarcity of 4D training data, which leads to low-quality video generation.
2) When using pose-driven 2D video models, such as PERSONA~\cite{sim2025persona}, the generated videos suffer from identity loss, as the model hallucinates a different appearance for each video.}

To address these issues, we propose \MethodName, a framework that marries kinematics-based animation with 2D video diffusion priors. 
We first design a layered motion representation that adopts mesh rigging as a strong motion prior, augmented by a deformation field~\cite{cao2023hexplane, Wu_2024_CVPR} for modeling non-rigid motion.
In contrast to direct reconstruction methods~\cite{sim2025persona, xie2024sv4d, yao2024sv4d2}, our key insight is to leverage kinematics as a strong structural and identity prior. 
Specifically, we first render a coarse video from our mesh-rigged animation, which provides a powerful, view-consistent constraint on the human's identity and structure that previous methods lack.
We then use a pretrained video diffusion model to restore this rendering, tasking it to synthesize realistic non-rigid dynamics (\emph{e.g.}, clothing flow) onto the existing structure rather than inventing an identity. 
These restored videos provide high-fidelity photorealistic supervision to optimize the residual motion field.

The restoration, however, is highly non-trivial. 
The initial renderings are unrealistic and thus out-of-distribution (OOD) for the pretrained video model.
\todo{Framing this restoration as a diffusion sampling task}, we find that standard deterministic ODE sampling methods fail on this OOD data, producing unsatisfying results (detailed in \cref{fig:mismatch}). Therefore, our core technical contribution is self-guided stochastic sampling, a novel restoration method designed for this task. 
Motivated by the robustness of stochastic sampling~\cite{Karras2022edm, ma2024sit} in correcting OOD samples, we develop a stochastic counterpart for deterministic flow matching. 
To solve the identity loss that occurs during this high-noise restoration process, we further introduce a self-guidance mechanism. Inspired by posterior sampling (DPS)~\cite{chung2023diffusion}, this guidance modifies the sampling process to ensure the posterior mean remains faithful to the input in preserved regions.

Finally, to robustly use these high-fidelity restored videos for 4D optimization, we must account for the inherent inconsistency across multiple samples. We employ diagonal view-time sampling as an efficient strategy to provide a coherent optimization signal by minimizing the number of generative trajectories, enabling sharp reconstruction.

\todo{In summary, \MethodName achieves photorealistic human animation results with the novel self-guided stochastic sampling algorithm, generates high-fidelity photorealistic non-rigid dynamics, capturing the natural flow of the dress (as shown in \cref{fig:teaser}), significantly surpassing the state-of-the-art methods.}
We provide extensive analysis and ablations that validate the critical roles of both stochasticity (for photorealistic quality) and self-guidance (for identity fidelity) in our sampler. 
These experiments demonstrate that sampling method is an essential and effective technique for restoring OOD renderings into high-quality, identity-preserving videos for 4D supervision.

\begin{figure*}
    \centering
    \includegraphics[width=\linewidth]{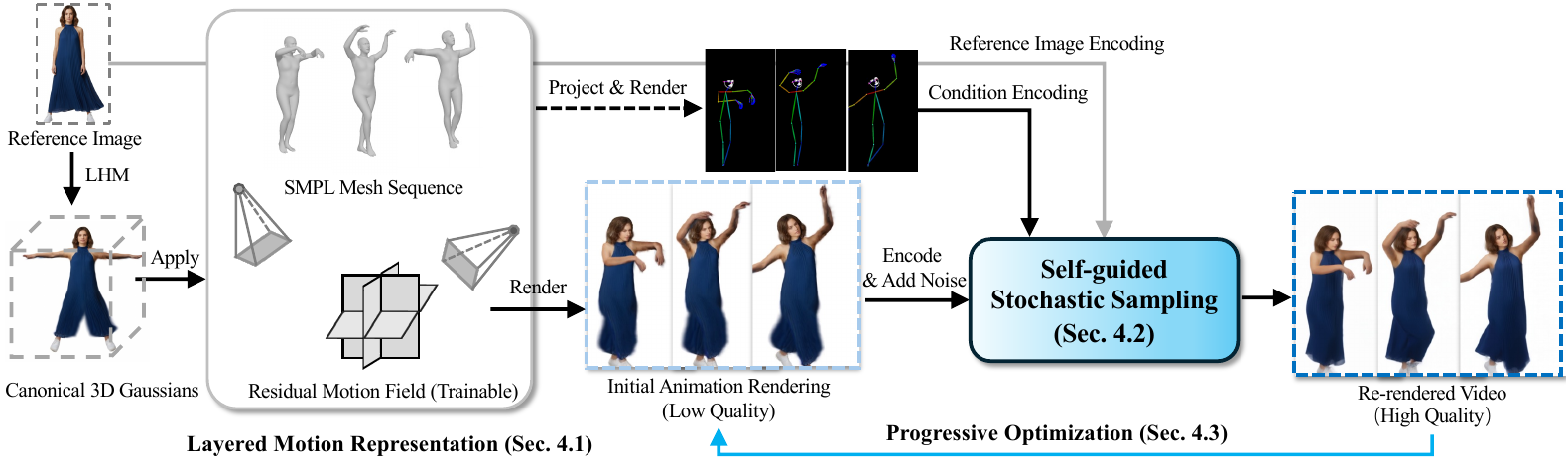}
    \caption{\textbf{Pipeline overview.}
    Our \MethodName animates a 3D Gaussian $\mathcal G$ (reconstructed with LHM~\cite{qiu2025LHM} from the reference image) with a mesh sequence.
    Our layered motion combines a mesh-rigged motion with a residual field for non-rigid dynamics.
    A coarse rendering $\vy$ from the rigid motion is restored to a high-quality video $\vx^*$ using our self-guided stochastic sampling.
    This restored video $\vx^*$ then provides supervision to progressively optimize the residual motion field.
    }
    \label{fig:framework}
    \vspace{-4mm}
\end{figure*}

\section{Related Work}
\label{sec:related_works}

\subsection{Traditional 3D Human Animation}
\paragraph{Kinematics-based methods.}
Kinematics-based methods~\cite{rekik2022} efficiently driving character motion by controlling skeletal poses.
Among these, Linear Blend Skinning (LBS)~\cite{lbs} is a core and widely used technique, deforming the surface mesh through a weighted average of bone transformations.
This paradigm was significantly advanced by parametric models like SMPL~\cite{SMPL:2015, SMPL-X:2019}, which extends LBS with identity-driven shape variations and pose-dependent shape variations.
Such models are pivotal for tasks like motion retargeting, adapting existing motions to new characters.
This reliance on an explicit, mesh-driven structure continues in current human implicit field and 3DGS methods~\cite{hugs, qian20243dgs, qiu2025LHM, liu2023humangaussian, jiang2023avatarcraft, ARAH:ECCV:2022}, which typically achieve animation via their corresponding meshes.
While the rendering quality of explicit meshes may lag behind modern video generation, this mesh-based approach provides a significant (rigid) motion prior. We leverage this by incorporating the mesh-rigged motion into our layered motion.

\paragraph{Physics-based animation.}
Another line of research~\cite{su2023caphy, PhysAavatar24, grigorev2023hood} uses physics simulation to enhance visual realism, particularly for modeling the complex dynamic effects of clothing in interaction with human bodies. 
These methods often require modeling the garment as a separate mesh to simulate its physical interactions with the body. 
For instance, PhysAvatar~\cite{PhysAavatar24} adopts the Codimensional Incremental Potential Contact (C-IPC) solver~\cite{Li2021CIPC}, which uses a log-barrier penalty function for its robustness in handling complicated body–cloth collisions.
However, this pipeline demands heavy runtime cost and extensive preprocessing, including the creation of separate garment meshes and the meticulous tuning of numerous physical parameters (\emph{e.g.}, stiffness, damping). 
To avoid this modeling and simulation complexity, we use a general motion field to represent the complex non-rigid deformation, and video diffusion prior for effective supervision.

\subsection{Video Diffusion Prior for 3D Animation}
\paragraph{Score distillation sampling (SDS).}
Recent advances in video diffusion models~\cite{wan2025, videoworldsimulators2024, HaCohen2024LTXVideo, kong2024hunyuanvideo, hong2022cogvideo, yang2024cogvideox, he2022lvdm, xing2023dynamicrafter, chen2023videocrafter1, chen2024videocrafter2} have inspired research on distilling 4D dynamic scenes from pre-trained models. MAV3D~\cite{singer2023text} was an early text-to-dynamic object work using a hexplane representation. 
Following methods leverage 3D Gaussian Splatting for high-fidelity rendering~\cite{li2024dreammesh4d, ling2023align, wimmer2025gaussianstolife, bahmani2024tc4d, li2025akd, sun2025animus3d}.
Methods like DG4D~\cite{ren2023dreamgaussian4d} and Disco4D~\cite{pang2025disco4d} use single-view videos as supervision alongside SDS with 3D-aware image diffusion to enhance unseen views.
However, SDS is known to suffer from lengthy optimization times~\cite{tang2023dreamgaussian}, unstable training dynamics~\cite{hertz2023delta}, and unsatisfactory generation quality, such as oversmoothing or over-saturation~\cite{mcallister2024rethinking}.

\paragraph{Photometric reconstruction with generated videos.} 
Another line of research in 3D animation is photometric reconstruction from generated videos, optimizing a generic 4D representation without cumbersome SDS objective.
This strategy primarily follows two paths.
The first employs multi-view video diffusion models~\cite{xie2024sv4d, yao2024sv4d2, li2024vividzoo, zhang20244diffusion, liang2024diffusion4d, wang20254real, sun2024eg4d, wu2025cat4d, shao2024human4dit, jiang2024animate3d}, which are fine-tuned to synchronize video generation across views.
To address human animation specifically, CharacterShot~\cite{gao2025charactershot} enhances image-to-video diffusion transformer~\cite{yang2024cogvideox} with 2D pose conditions and extends to multi-view setting.
While conceptually sound, these models are fundamentally limited by the scarcity of high-quality 4D training data, and their generation quality lags significantly behind general 2D video models.
The second strategy, used by PERSONA~\cite{sim2025persona}, employs pose-controlled 2D video diffusion~\cite{zhang2025mimicmotion} to create training data.
It then uses these generated videos to jointly optimize canonical 3D Gaussians and pose-dependent non-rigid deformation field, finally adopting LBS for animation.
While this leverages high-quality 2D priors, it suffers from severe identity loss and temporal shifts, as the model hallucinates a different appearance for each generated video.
In contrast to these direct reconstruction, our method restores initial mesh-rigged renderings, which provides a strong structural and identity prior.
And our self-guided stochastic sampling ensures both photorealistic quality and strong fidelity.

\section{Preliminary: Flow Matching}
Building upon the success of denoising diffusion models~\cite{ho2020denoising, song2021scorebased}, Flow Matching (FM)~\cite{liu2022flow, lipman2022flow} generates samples by learning a velocity field $\vv_\theta$ that transports a prior distribution $p_1$ to the data distribution $p_0$. Rectified Flow~\cite{liu2022flow}, a notable variant, simplifies this by defining a linear interpolation path: $\vx_t = (1-\sigma_t)\vx_0 + \sigma_t \vx_1$, where $\vx_0 \sim p(\mathbf{x})$ is a data sample and $\vx_1 \sim \mathcal{N}(0, 1)$ is a noise sample.
The corresponding target velocity field is constant, $\vu_t = \vx_1 - \vx_0$. The model $\vv_\theta$ is trained to predict this velocity with the objective:
\begin{equation}
\min_\theta \mathbb E_{t, \vx_0, \vx_1} ||\vv_\theta (\vx_t, t) - (\vx_1 - \vx_0) ||_2^2.
\label{eq:fm_objective}
\end{equation}
The sampling process starts from $\vx_1 \sim \mathcal{N}(0, 1)$ and integrates the learned field $\vv_\theta$ backward using an ODE solver:
\begin{equation}
{d \vx_t} = \vv_\theta(\vx_t, t) dt, \quad \text{solved from } t=1 \text{ to } t=0.
\label{eq:pfode}
\end{equation}
A key property, analogous to the Tweedie's formula~\cite{Tweedie}, is the ability to predict the path's endpoints ($\vx_0$ and $\vx_1$) from any intermediate point $\vx_t$. By rearranging the path definition and substituting $\vv_\theta \approx \vu_t$, we can derive estimators for both the posterior mean ($\hat{\vx}_0$, the predicted data) and the posterior noise ($\hat{\vx}_1$):
\begin{align}
\vu_t = \frac{\vx_t - \vx_0}{\sigma_t} &\implies \hat{\vx}_{0|t} = \vx_t - \sigma_t \vv_\theta(\vx_t, t); \label{eq:posterior_mean} \\
\vu_t = \frac{\vx_1 - \vx_t}{1-\sigma_t} &\implies \hat{\vx}_{1|t} = \vx_t + (1-\sigma_t) \vv_\theta(\vx_t, t). \label{eq:posterior_noise}
\end{align}
A standard deterministic (ODE) solver uses these two predictions to perform an update step (from $t$ to $t_\text{next}$) by re-interpolating on the linear trajectory:
\begin{equation}
\vx_{t_\text{next}} = (1-\sigma_{t_\text{next}})\hat{\vx}_{0|t} + \sigma_{t_\text{next}}\hat{\vx}_{1|t}.
\label{eq:ode_update}
\end{equation}

\section{Proposed Method}
\label{sec:method}

We present the pipeline of our framework in \cref{fig:framework}. Given a 3D human $\mathcal{G}$ represented by 3DGS~\cite{kerbl3Dgaussians} (reconstructed from reference image) and a mesh sequence represented by the SMPL parameter sequence $\{s_t\}_{t=1}^N$,our goal is to animate the 3D human by modeling both rigid and non-rigid motions, such as body pose and cloth deformations, and to render a photorealistic video from any viewpoint. This is achieved by first defining a layered motion representation (\textbf{Sec. 4.1}) and then proposing a self-guided stochastic sampling approach (\textbf{Sec. 4.2}) to generate high-quality video supervision signals for progressive optimization (\textbf{Sec. 4.3}) to learn the motion field.

\subsection{Layered Motion Representation}

Our human motion representation combines an explicit mesh-rigged motion with an implicit residual motion field.

\paragraph{Mesh-rigged motion.} 
We first model the rigid motion using a mesh-rigged approach based on SMPL~\cite{SMPL:2015, SMPL-X:2019}. 
SMPL describes articulated motion via a sparse set of skeleton parameters $\{s_\tau\}_{\tau=1}^T$.
To drive our 3DGS sequence, we establish a bijective correspondence between them and a point cloud $\{p_i\}_{i=1}^N$ on the surface of the canonical SMPL-X mesh~\cite{SMPL-X:2019}. 
This mapping is based on spatial relations (\emph{e.g.}, Euclidean distance or SDF).
At each time step $\tau$, the skeleton parameters $s_\tau$ determine the translation and rotation of each SMPL point $p_i$. We then apply these exact transformations to the corresponding Gaussian $\mathcal{G}_i$, thus animating the rigid motion of the 3D Gaussians.

\paragraph{Residual motion field.}
Since the mesh-rigged motion field cannot capture non-rigid motions, we incorporate a residual motion field to model the non-rigid motion. This implicit function is parameterized with a Hexplane~\cite{cao2023hexplane, Wu_2024_CVPR}, 
which first queries feature $f_p$ in the canonical Gaussian position $\vp$.
Once the feature is obtained, a lightweight decoder implemented with an MLP predicts the offset $\Delta \theta$ of the Gaussian parameters $\theta$, such as position and rotation.

\begin{figure}
    \centering
    \includegraphics[width=\linewidth]{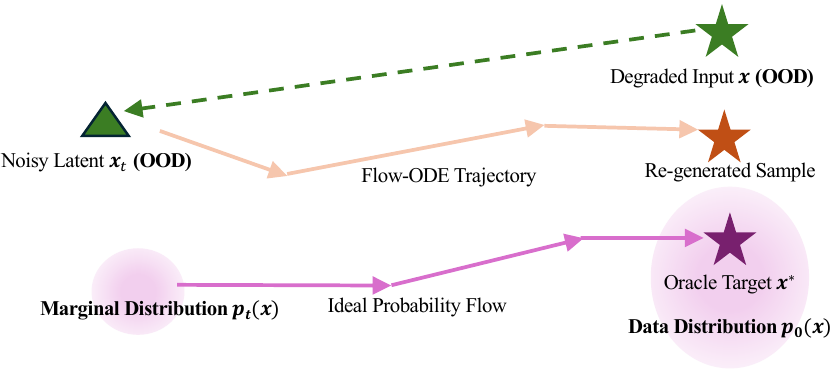}
    \caption{\textbf{Distribution mismatch in deterministic flow matching.}
    Our degraded input $\vy$ (out-of-distribution, OOD) creates a noisy latent $\vx_t$ that is off the marginal distribution $p_t(\vx)$.
    A deterministic Flow-ODE (\textcolor{orange}{orange path}) follows an incorrect trajectory as its velocity predictions are inaccurate for OOD samples, resulting in a low-quality sample.
    This motivates our use of an SDE sampler, which can actively correct the path by driving the sample back toward the marginal distribution.
    }
    \label{fig:mismatch}
    \vspace{-4mm}
\end{figure}

\subsection{Video Re-rendering with Self-guided Stochastic Flow Sampling}
\label{subsec:sampling}

Our goal is to generate a high-quality, identity-preserving video $\vx^*$ to provide effective supervision for the residual motion field.
However, the initial video $\vy$, rendered by 3DGS from the mesh-rigged motion, is highly unrealistic, exhibiting artifacts such as missing (garment) regions and unstable fine-grained motion. 
Therefore, we propose self-guided stochastic sampling to restore this coarse input $\vy$ into the high-fidelity target $\vx^*$.

\paragraph{Limitation of deterministic flow-ODE.}
We formulate the re-rendering as a video-conditioned sampling problem.
Following SDEdit~\cite{meng2022sdedit}, we first inject significant Gaussian noise into the input $\vy$ to a level $t$:
\begin{equation}\vx_{t} = \sigma_{t} \epsilon + (1-\sigma_{t})\vy, \quad \epsilon \sim \mathcal N(0,1). \label{eq:noise_inject}\end{equation}
A baseline approach would be to use a deterministic ODE solver (Flow-ODE)~\cite{liu2022flow} to reverse this process from $t$ to $0$ (see \cref{eq:pfode}). However, this approach fails. As shown in~\cref{fig:mismatch}, since the input $\vy$ is OOD, its noised version $\vx_{t}$ also lies off the marginal distribution $p_t(\vx)$ that the flow model was trained on. A deterministic ODE trajectory has no mechanism to correct this error; it is ``off the rails" and will follow an incorrect path, leading to low-quality results.

\paragraph{High-quality generation with stochastic sampling.}
Stochastic SDE sampling typically yields higher generation quality \cite{Karras2022edm, song2021scorebased, nie2023blessing, singh2024stochastic} compared to ODE sampling. 
More importantly, EDM \cite{Karras2022edm} proves that the stochastic process actively pulls samples toward the target marginal $p_t(\vx)$ at each step, correcting errors accumulated from the initial OOD state.
Motivated by this, we propose a reverse-time SDE analogous to the deterministic flow ODE in \cref{eq:pfode}:
\begin{equation}
    d \vx = \vv_t (\vx_t, t)dt + g(t) d\vw_t, \label{eq:sde}
\end{equation}
where $g(t)d\vw_t$ is the stochastic ``diffusion" term that performs the correction.
To implement this, we propose a novel stochastic discretization: we add noise directly to the ``clean noise" prediction $\hat{\vx}_{1|t}$ before interpolation: \begin{equation}
    \hat \vx_{1|t} \leftarrow \sqrt{\gamma (t)} \epsilon + \sqrt{1-\gamma (t)} \hat \vx_{1|t},
    \label{eq:renoise}
\end{equation}
where $\gamma(t)$ is set to $\sigma_t$ empirically. 
This re-noising (Eq.~\ref{eq:renoise}) is our simple and effective implementation of the stochastic term $g(t)d\vw_t$, providing the necessary path correction, which is essential for achieving high-quality results from our OOD inputs.
Proof is detailed in the \emph{supplementary}.

\begin{algorithm}[t]
\caption{Self-guided Stochastic Sampling (Practical)}\label{alg:m}
\begin{algorithmic}[1]
\Require Low-quality video $\vy$; Pre-trained flow-based model $\vv_\theta$; Preserved region $\mathcal M$; Initial noise step $t_0$, constant step size $\lambda$;
\Ensure \wangcan{Desirable high-quality video $\vx^*$}
\State \textbf{Sample} $\epsilon \sim \mathcal{N}(0, \rmI)$
\State $\x_t = \sigma_{t_0}\epsilon + (1-\sigma_{t_0}) \vy$
\Comment{\cref{eq:noise_inject}}
\For{$t: t_0\rightarrow 0$} \Comment{Sampling loop}
    \State $\hat \vx_{0|t} \gets \vx_t - \sigma_t\vv_\theta(\vx_t, t)$ \Comment{\cref{eq:posterior_mean}}
    \State $\hat \vx_{1|t} \gets \vx_t + (1-\sigma_t)\vv_\theta(\vx_t, t)$ \Comment{\cref{eq:posterior_noise}}
    \State $\hat \vx_{0|t} \leftarrow \hat \vx_{0|t} - \lambda \nabla_{ \vx_{t}} ||\mathcal M \odot (\vy - \hat \vx_{0|t} )||^2$  \Comment{\cref{eq:dps}}
    \State \textbf{Sample} $\epsilon \sim \mathcal{N}(0, \rmI)$ 
    \State $\hat \vx_{1|t} \gets \sqrt{1-\sigma_{t}}\hat\vx_{1|t} + \sqrt{\sigma_{t}}\epsilon$ \Comment{\cref{eq:renoise}}
    \State $\vx_{t_\text{next}} \gets (1-\sigma_{t_\text{next}})\hat\vx_{0|t} + \sigma_{t_\text{next}} \hat \vx_{1|t}$ \Comment{\cref{eq:ode_update}}
\EndFor 
\State \textbf{Return} $\vx^*=\vx_t|_{t=0}$ 
\end{algorithmic}
\end{algorithm}


\paragraph{Identity preserving with self-guidance.}
The entire re-rendering process begins by injecting a high level of noise $t$ (Eq.~\ref{eq:noise_inject}) into the input $\vy$. This high noise, while necessary for the stochastic sampler to reset and find the correct data manifold, simultaneously destroys or corrupts identity-critical information from the original input.
Consequently, the unguided stochastic sampler, while producing a high-quality video, will fail to preserve the human's identity; it will hallucinate a plausible but incorrect appearance (\emph{e.g.}, a different face) that is consistent with the corrupted noisy latent $\vx_{t}$.
Therefore, to ensure fidelity to the original input $\vy$, we must explicitly guide the sampling process. Theoretically, this conditioning is achieved by modifying the SDE's drift term (Eq.~\ref{eq:sde}) with the score of the posterior $p(\vy|\vx_t)$: 
\begin{equation} d \vx = [\vv_t (\vx_t, t)-g(t)^2 \nabla_{\vx_t} \log p (\vy|\vx_t)]dt + g(t) d\vw_t. \label{eq:guided_sde} 
\end{equation}
However, the guidance term $\nabla_{\vx_t} \log p (\vy|\vx_t)$ is intractable. We therefore adopt the core insight from Diffusion Posterior Sampling (DPS)~\cite{chung2023diffusion}, which provides an elegant and practical approximation. DPS proves that this complex score-space guidance (Eq.~\ref{eq:guided_sde}) can be effectively approximated by applying a simple data-space L2 loss to the posterior mean $\hat{\vx}_{0|t}$.
In each sampling step, we first compute the standard posterior mean $\hat{\vx}_{0|t}$ (Eq.~\ref{eq:posterior_mean}). Then, we apply a guidance step to pull this prediction closer to our masked input $\vy$: 
\begin{equation} 
\hat \vx_{0|t} \leftarrow \hat \vx_{0|t} - \lambda (t) \nabla_{x_t} ||\mathcal M \odot (\vy - \hat \vx_{0|t} )||^2, \label{eq:dps} 
\end{equation}
where $\mathcal M$ is a binary mask for preserved regions (e.g., face, hands) and $\lambda(t)$ is the step size.
This gradient has a simple closed-form solution (as derived in the supplementary), making this step computationally efficient.

Finally, we combine our stochastic component (Eq.~\ref{eq:renoise}) and our new guided component (Eq.~\ref{eq:dps}) into the standard update rule (Eq.~\ref{eq:ode_update}) to compute the next sample $\vx_{{t_\text{next}}}$, \todo{finally obtain $\vx^*$}.
This self-guided stochastic sampler thus achieves both high-quality generation (from the SDE) and strong identity preservation (from the guidance). The practical algorithm is summarized in \cref{alg:m}.

\begin{figure}[t]
    \centering
    \includegraphics[width=\linewidth]{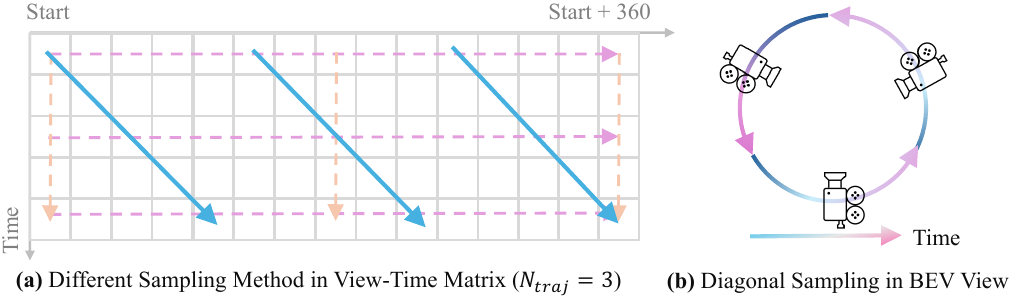}
    \caption{\textbf{Diagonal view-time sampling}. 
   (a) Illustration of \textcolor{cyan}{diagonal sampling} in a view-time matrix ($N_\text{traj}=3$). This method simultaneously evolves the camera view and time, distinct from fixed-time (\textcolor{magenta}{bullet-time}) or fixed-camera (\textcolor{orange}{independent-view}) sampling. (b) An example trajectory shows the camera orbiting 360° as time progresses.}
    \label{fig:prog}
    \vspace{-4mm}
\end{figure}

\begin{figure*}
    \centering
    \includegraphics[width=\linewidth]{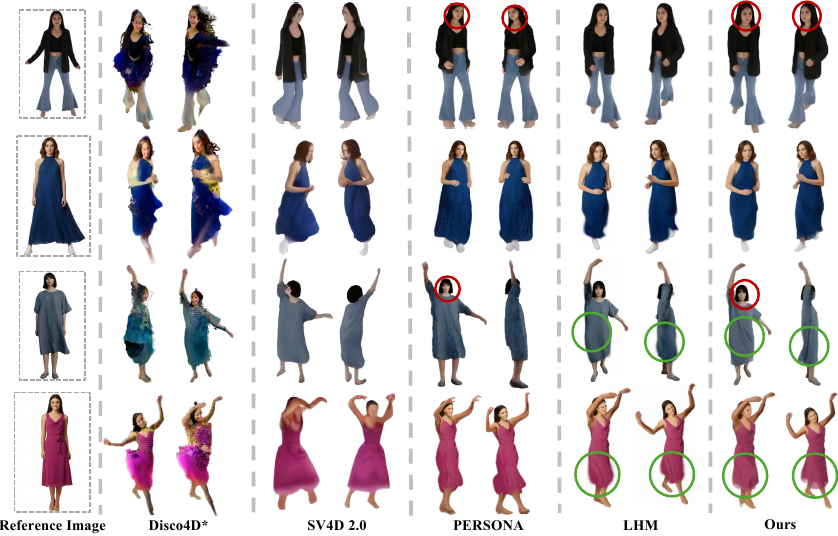}
\caption{\textbf{Comparison with state-of-the-art methods.} Our method (Ours) is the only one to simultaneously achieve high quality, identity preservation, and realistic non-rigid motion. Existing methods fail in key areas: Disco4D~\cite{pang2025disco4d} and SV4D 2.0~\cite{yao2024sv4d2} suffers from low quality (due to SDS and multi-view video diffusion);  PERSONA loses identity (due to direct reconstruction from pose-driven video diffusion); and LHM~\cite{qiu2025LHM} captures identity but fails to model clothing dynamics. (* self-implementation)}    \label{fig:sota}
    \vspace{-5mm}
\end{figure*}

\paragraph{Personalized diffusion prior.} 
While general video diffusion models are powerful, their priors may not be optimized for human animation. 
\wangcan{Therefore, we finetune a video model with two human-related conditions to provide personalized diffusion prior:}
(1) a reference human image via an additional branch, providing a strong prior for identity preservation (fidelity), and (2) a 2D pose sequence for precise motion control. By pre-training specifically on human-centric data, this specialized prior offers a more suitable foundation, enabling higher generation quality and photorealism compared to a general-purpose video model.

\subsection{Progressive 4D Optimization}

\paragraph{Diagonal view-time sampling.}
We use our high-quality restored videos to optimize the residual motion field. The primary challenge is the inter-trajectory inconsistency~\cite{liu2025free4d, wang2024vistadream,wu2025difix3d} of generative models. Standard aggregation of many independent-view~\cite{xie2024sv4d} or bullet-time~\cite{rozumnyi2025bulletgen} trajectories accumulates conflicting signals, leading to blurred artifacts. We therefore propose diagonal view-time sampling (sampling $v$ and $t$ simultaneously, \cref{fig:prog}), as it captures spatio-temporal information using the minimum number of trajectories, thus minimizing exposure to inconsistency.

\paragraph{Dataset update.}
To address the sparsity of this minimal set, we pair it with a progressive dataset update strategy~\cite{instructnerf2023}. Every $5k$ iterations, we generate new trajectories based on the current state of the 4D model and add them to the training set. This ``generation-optimization" cycle progressively densifies the supervision in a consistent manner, ensuring a high-fidelity 4D reconstruction.

\paragraph{Optimization objective.}
We adopt the commonly-used photometric loss: L1 loss, dSSIM loss and LPIPS loss to supervise the 4D representation.
In addition, to preserve the geometry of rigid part of motion, we design a regularization by calculating depth difference between the original depth of the optimized depth within the preserved region:
\begin{equation}
    \mathcal L = \mathcal L_{\text{L1}} + \lambda_1\mathcal L _{\text{LPIPS}} + \lambda_2 \mathcal L_{\text{dssim}} + \lambda_3 \mathcal L_{\text{mask}}  \\
    + \lambda_4 \mathcal L_{\text{reg}}.
\end{equation}

\begin{figure*}
    \centering
    \includegraphics[width=0.9\linewidth]{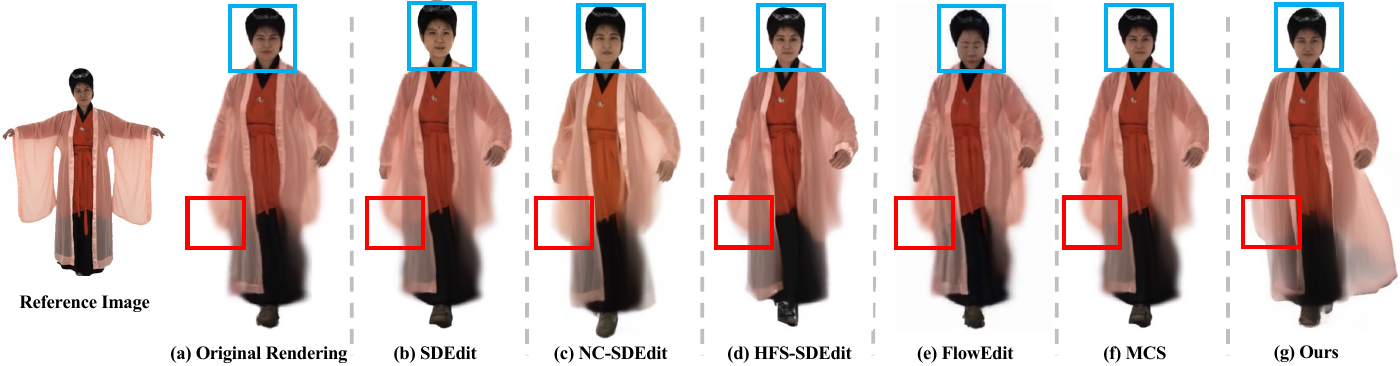}
    \caption{\textbf{Comparison on other video re-rendering methods.} (a) original rendering $\vx$; (b-f) competitive sampling methods; (g) our results $\vx^*$.
    Only our self-guided stochastic sampling can generate sharp details while preserving the original identity well.
    }  
    \vspace{-3mm}
    \label{fig:comp_rerender}
\end{figure*}
\begin{figure*}[ht]
    \centering
    \includegraphics[width=0.9\linewidth]{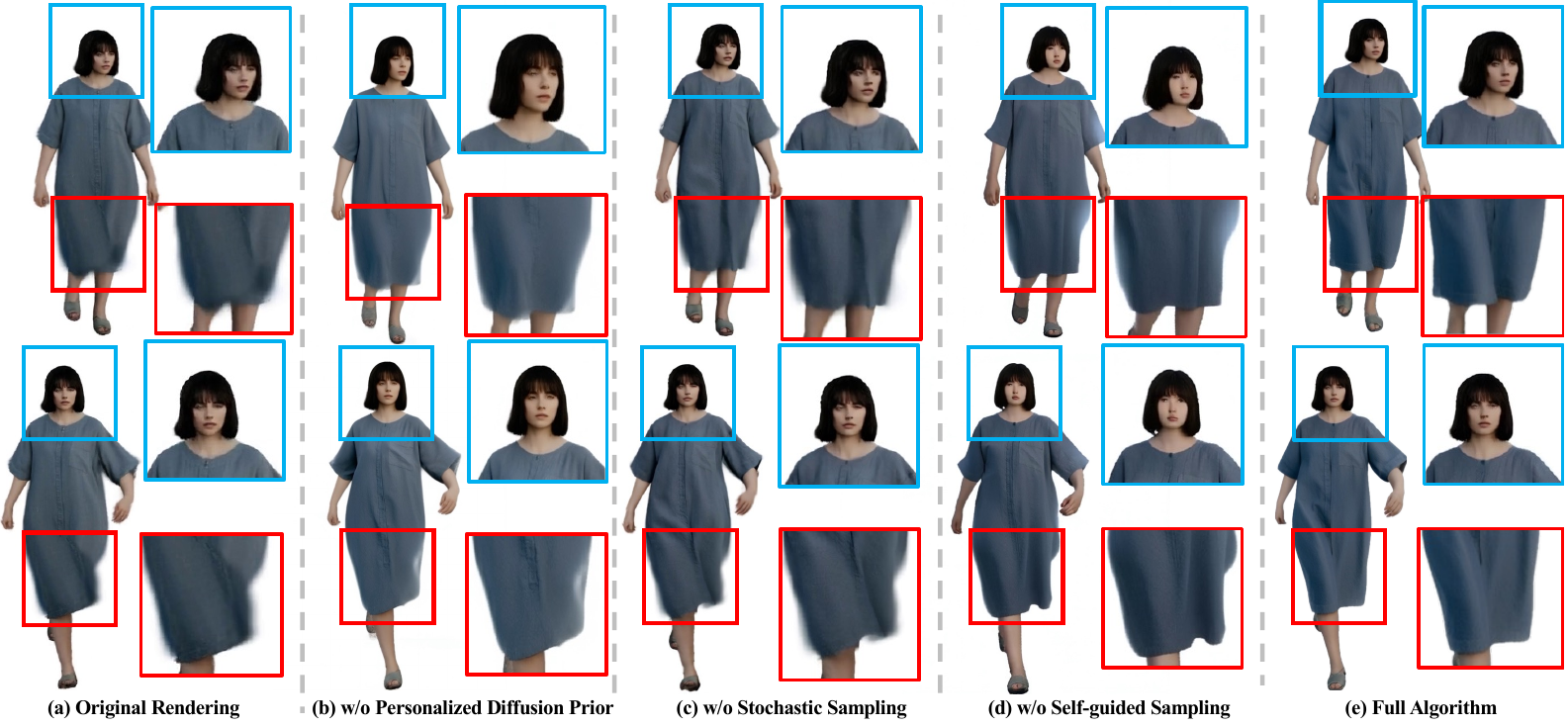}
    \caption{\textbf{Ablative experiment on self-guided stochastic sampling.} We compare full sampling method (e) and the generation results of our model with a set of ablations. (a) Original rendering with mesh-rigged animation; (b) replacing personalized diffusion prior with general diffusion prior introduces slight performance degradation and artifacts; (c) We observe that our method produces significant quality drop when removing stochastic sampling; (d) removing self-guided sampling greatly reduces the identity preservation.
    }
    \vspace{-4mm}
    \label{fig:abl_rerender}
\end{figure*}

\section{Experiments}
\label{sec:Experiments}
\subsection{Experimental settings}

\paragraph{Implementation details.}
For fair comparison, we use LHM~\cite{qiu2025LHM} to generate the canonical 3D human Gaussians from a single-view image.
The render resolution is set to $H/W/T$=832/480/81.
We finetune the personalized pose-conditioned video diffusion based Wan2.1-1.3B~\cite{wan2025}.
For key parameter in video re-rendering, we set noise injection rate $t_0=0.6$, initial denoising step $N=30$, $\lambda=0.2$ in self-guidance.
Preserved mask is obtained with SAM2~\cite{ravi2024sam2}.
All experiments are conducted on a NVIDIA A6000 48G GPU.
In progressive 4D optimization, we use AdamW~\cite{kingma2014adam} with constant learning of 1e-5 with $25k$ iterations.

\paragraph{Baselines.}
We compare our method with several state-of-the-art 3D animation methods: LHM~\cite{qiu2025LHM}, Disco4D~\cite{pang2025disco4d}, SV4D 2.0~\cite{yao2024sv4d2}, and PERSONA~\cite{sim2025persona}.
LHM reconstructs human Gaussian from single image, then animate the human Gaussian with mesh-rigged motion.
Disco4D uses hybrid supervision of single-view video and 3D-aware image SDS to optimize the motion field.
SV4D 2.0 first generates multi-view videos given single-view video then uses DynNeRF~\cite{Gao-ICCV-DynNeRF} for reconstruction.
PERSONA learns a pose-dependent non-rigid deformation (relative to canonical SMPL) from pose-conditioned video diffusion, then uses LBS to animate the Gaussian particles.

\paragraph{Evaluation tasks, dataset, and metrics.}
To evaluate the performance quantitatively, we select 10 cases from ActorsHQ~\cite{isik2023humanrf} dataset, and reconstruct from single-view image and extracted motion sequence. 
Then  we use common pixel-wise reconstruction metrics, PNSR, SSIM, LPIPS, and CLIP-Image~\cite{radford2021learning} to measure the similarity to the ground truth, and adopt FID/FVD to evaluate the general rendered image/video quality.
For novel motion (no GT), we conduct a user study to evaluate the quality in identity preservation, frame quality, motion realism, physical plausibility in non-rigid part and overall preference.
\begin{table}[t]
    \centering
    \vspace{-1mm}
        \caption{\textbf{Quantitative results with the state-of-the-art methods of human animation in ActorsHQ~\cite{isik2023humanrf} dataset.}}
         \resizebox{0.8\linewidth}{!}{
    \begin{tabular}{lccccccc}
    \toprule 
        Methods & PSNR $\uparrow$ & SSIM $\uparrow$ & LPIPS $\downarrow$ & CLIP-I $\uparrow$ & FID $\downarrow$  & FVD $\downarrow$ \\
        \midrule
        Disco4D~\cite{pang2025disco4d} & 12.05 & 0.5590 & 0.5019 & 0.6439 & 613.9 & 622.1\\
        SV4D 2.0~\cite{yao2024sv4d2} & 15.25 & 0.7708 & 0.3773 & 0.7640 & 364.9 & 478.7 \\
        PERSONA~\cite{sim2025persona} & 17.01 & 0.8219  & 0.2602 & 0.8779 & 199.1 &  367.0 \\ 
        LHM~\cite{qiu2025LHM} & 19.51 &\textbf{0.8382 } &  0.2169 & 0.9009 & 124.1 & 339.9 \\
        Ours & \textbf{20.08} & 0.8312 & \textbf{0.2125}  & \textbf{0.9160} & \textbf{105.3} & \textbf{295.2} \\
    \bottomrule
    \end{tabular}
    }
    \vspace{-4mm}
    \label{tab:sota_compa}
\end{table}

\begin{table}[t]
    \centering
    \caption{\textbf{User study} on human animation with novel motion.}
    \label{tab:user}
    \resizebox{0.9\linewidth}{!}{
    \begin{tabular}{lcccccc}
    \toprule
    \multicolumn{1}{c}{Methods ($\%$)} & \makecell[c]{Identity \\ Preservation} & \makecell[c]{Frame \\ Quality} & \makecell[c]{Motion \\ Realism} &  \makecell[c]{(Non-rigid) Physical \\ Plausibility} & \makecell[c]{Overall \\ Preference} \\
    \midrule
    Disco4D~\cite{pang2025disco4d} & 0 & 1.47 & 4.41 & 0 & 0 \\
    SV4D 2.0~\cite{yao2024sv4d2} & 0& 2.94 & 16.2 & 4.34 & 5.89 \\
    PERSONA~\cite{sim2025persona} & 20.6 & 30.9 & 14.7 & 19.1&  22.1\\
    LHM~\cite{qiu2025LHM} & 39.2& 29.4 & 29.4 & 14.7 & 17.6\\
    Ours & \textbf{40.1} & \textbf{35.3} & \textbf{35.3} & \textbf{61.8} & \textbf{54.4}\\
    \bottomrule
    \end{tabular}
    }
    \vspace{-5mm}
\end{table}

\subsection{Comparisons to the State-of-the-art Methods}
We compare our method and four the state-of-the-art methods in~\cref{fig:sota}.
In this comparison, we demonstrate 4 human with different motion sequence.
Disco4D suffers from over-saturation and notable artifacts due to SDS.
Directly reconstructing from multi-view diffusion videos (SV4D) cannot achieve high quality.
Although equipped with lots of preprocessing and regularization, PERSONA cannot preserve the original identity well.
While LHM that applies mesh-rigged motion that achieves high identity preservation and precise body control, but it cannot model the realistic non-rigid garment motions.
Only our method can generate photorealistic and physics-plausible non-rigid motion in various motions, such as clothing dynamics like {fluttering and folds} (also see \cref{fig:teaser}).
Quantitative results in \cref{tab:sota_compa} also support this, we have surpassed the state-of-the-art methods, with competitive reconstruction metrics and considerable 18.8 FID improvement.
User preference also shows that our method shares the best scores among all terms.

\subsection{Analysis of Self-guided Stochastic Sampling}

\paragraph{Comparison with other sampling methods.}
We compare our self-guided stochastic sampling with several competitive methods in~\cref{fig:comp_rerender}, including vanilla SDEdit-FM~\cite{meng2022sdedit}, FlowEdit~\cite{kulikov2024flowedit}, MCS~\cite{wang2024vistadream}, HFS-SDEdit~\cite{elevating3d}, and NC-SDEdit~\cite{yang2024noise}. For a fair comparison, all methods use the same base video model, an initial noise level of $t_0=0.6$, and 30 denoising steps.
Vanilla SDEdit \textbf{(b)} fails to preserve the human's identity.
To address this, following works in visual editing/restoration incorporate the input video $\vx$ to improve fidelity.
For example, MCS~\cite{wang2024vistadream} \textbf{(f)} updates the posterior mean  $\hat \vx_{0|t}$  with the weighted averaging of itself and the reference latents $\vx$.  
HFS-SDEdit~\cite{elevating3d} \textbf{(d)} matches the high-frequency component of posterior mean with that of the reference image $\vx$.
However, these methods~\textbf{(b-f)} are built on deterministic ODE sampling. As argued in~\cref{subsec:sampling}, ODE sampling struggles with the out-of-distribution (OOD) nature of our input, resulting in low-quality outputs. 
In contrast, our method \textbf{(g)} provides an effective solution that simultaneously achieves high-quality results (via stochastic sampling) and strong fidelity (via self-guidance), fully leveraging the power of the video diffusion prior. Implementation details are in the \emph{supplementary}.

\paragraph{Component-wise validation.}
We also validate the components of our method in~\cref{fig:abl_rerender}.
The original rendering \textbf{(a)} suffers from severe blur and unrealistic artifacts on the garment, caused by the initial mesh-rigged motion.
We ablate our two key contributions. First, in \textbf{(c)}, we replace our stochastic sampler with its deterministic ODE counterpart. The resulting video quality is significantly worse, and the blurriness persists, which confirms our hypothesis that stochastic sampling is essential for correcting the OOD input and achieving high-quality restoration. Second, in \textbf{(d)}, we remove the self-guidance term. While the video quality is high (due to stochastic sampling), the human's identity is lost. This demonstrates that our self-guidance is crucial for fidelity. Additionally, we replace our personalized diffusion prior with a general one \textbf{(b)}, which leads to a slight drop in realism. Our full method \textbf{(e)} is the only setting that successfully resolves the initial artifacts, generates a high-quality video, and faithfully preserves the human's identity.

\begin{figure}[t]
    \centering
    \includegraphics[width=\linewidth]{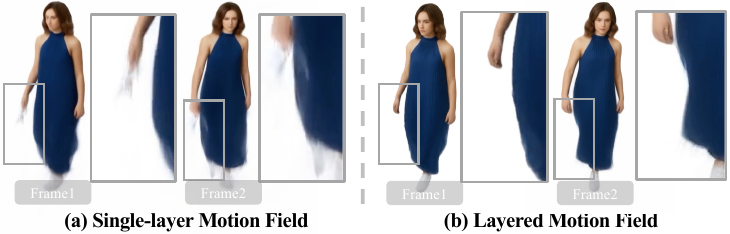}
    \caption{\textbf{Ablation study on motion field.}
    Our layered motion (right) captures intricate hand details, while the single-layer baseline (left) fails.}  
    \label{fig:abl_others}
    \vspace{-4mm}
\end{figure}

\begin{figure}[t]
    \centering
    \includegraphics[width=\linewidth]{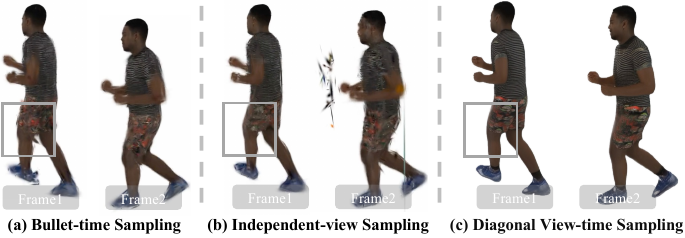}
    \caption{\textbf{Ablation study on sampling method.} Baseline methods (left) suffer from significant floaters and spikes, while our diagonal sampling (right) reconstructs sharp details.}  
    \label{fig:abl_others2}
    \vspace{-3mm}
\end{figure}

\subsection{\wangcan{Analysis of Motion Field and Data Sampling}}
We ablate our proposed methods: the motion representation and the view-time sampling strategy. 
Additional ablations on loss functions, progressive optimization, hyperparameter selection are detailed in the \emph{supplementary materials}.

\paragraph{Impact of layered motion representation.}
In \cref{fig:abl_others}, we compare our layered representation against single-layer motion field~\cite{liu2025free4d, yu20244real} initialized with mesh-rigged motion.
The baseline fails to model intricate transformations, such as human hands, whereas our layered approach captures these details effectively.

\paragraph{Impact of diagonal view-time sampling.}
In \cref{fig:abl_others2}, we compare our diagonal view-time sampling against bullet-time sampling and independent view sampling, using an equal trajectory count ($N_\text{traj}=3$). 
The baseline methods, suffering from temporal and spatial sparsity, produce significant floaters and spikes. In contrast, our strategy yields sharp, artifact-free details.

\section{Conclusion}
\label{sec:Conclusion}
We presented \MethodName, a novel framework for photorealistic human animation that successfully captures complex non-rigid motion.
First, we introduce a layered motion representation composed of a mesh-rigged motion and a residual field. 
Second, to supervise this, we propose self-guided stochastic sampling, which is specifically designed to transform our low-quality, out-of-distribution initial renderings into high-fidelity, identity-preserving videos. 
It achieves this by balancing stochasticity (for quality) with self-guidance (for fidelity).
We also introduce diagonal view-time sampling to ensure a coherent 4D optimization free from generative inconsistencies.
Comparative experiments show that our framework surpasses state-of-the-art methods, achieving best perceptual results. 
Our extensive ablation studies validate the effectiveness of our core sampling algorithm, as well as the necessity of the layered motion representation and diagonal sampling.
The key limitation lies in the lengthy sampling time of the video diffusion prior. 
A valuable future avenue is incorporating few-step generation techniques~\cite{huang2025selfforcing} to reduce the overall time cost.

{\small
\bibliographystyle{ieee_fullname}
\bibliography{ref}
}

\clearpage
\onecolumn
\appendix
\section*{Contents}
\startcontents[sections]
\printcontents[sections]{}{1}{}
\vspace{1cm}
\hrule
\vspace{1cm}
\section{Supplementary Video}

To better demonstrate the efficacy of our framework and the visual quality of our results, we provide a comprehensive supplementary video (overall length \texttt{2'47''}). 
We strongly recommend viewing the video to fully dynamic visual results.

\section{Proof}
\begin{proposition}[Error Bound of Gradient Approximation]
Consider the score approximation $\nabla_{\mathbf{x}_t} \log p(\mathbf{y} | \mathbf{x}_t) \approx \nabla_{\mathbf{x}_t} \log p(\mathbf{y} | \hat{\mathbf{x}}_{0|t})$ used in Eq.(10). Let $\mathcal{M}$ be the measurement operator and $\hat{\mathbf{x}}_{0|t} = \mathbb{E}[\mathbf{x}_0 | \mathbf{x}_t]$ be the posterior mean. Under the manifold constraint, the approximation error $\epsilon$ is upper bounded by:
\begin{equation}
    \epsilon \leq C \cdot \|\mathcal{M}\|_2 \cdot \mathbb{E}_{\mathbf{x}_0 \sim p(\mathbf{x}_0|\mathbf{x}_t)} [\|\mathbf{x}_0 - \hat{\mathbf{x}}_{0|t}\|],
\end{equation}
where $C$ is a constant related to the Lipschitz property of the noise schedule.
\end{proposition}

\begin{proof}
Following the theoretical framework in DPS~\cite{chung2023diffusion}, the likelihood gradient can be decomposed via the Tweedie’s formula. The spectral norm $\|\mathcal{M}\|_2$ represents the maximum amplification factor of the measurement operator. 

In our specific task, the operator $\mathcal{M}$ is defined as a binary mask $\mathbf{M} \in \{0, 1\}^n$. The spectral norm of a diagonal matrix (or masking operator) is given by its maximum singular value:
\begin{equation}
    \|\mathcal{M}\|_2 = \max_{i} |M_{ii}| = 1.
\end{equation}
Consequently, the error bound simplifies to $\epsilon \leq C \cdot \mathbb{E}[\|\mathbf{x}_0 - \hat{\mathbf{x}}_{0|t}\|]$. This term represents the uncertainty of the posterior estimation at time $t$. As the diffusion process approaches the clean data manifold ($t \to 0$), the posterior distribution $p(\mathbf{x}_0|\mathbf{x}_t)$ collapses to a Dirac delta distribution $\delta(\mathbf{x}_0 - \hat{\mathbf{x}}_{0})$, leading to $\epsilon \to 0$. This ensures that the approximate gradient converges to the true score direction in the final sampling stages.
\end{proof}

\begin{proposition}[SDE Correction Mechanism~\cite{Karras2022edm}]
The continuous implicit Langevin diffusion $d\vx_t = \frac{1}{2} \nabla \log p(\vx) dt + d\mathbf{w}_t$ actively corrects sampling errors by admitting the data marginal $p(\vx)$ as its unique stationary distribution.
\end{proposition}

\begin{proof}
The time evolution of the probability density $p_t(\vx)$ is governed by the Fokker-Planck Equation (FPE):
\begin{equation}
    \frac{\partial p_t}{\partial t} = -\nabla \cdot \left( \frac{1}{2} (\nabla \log p) p_t \right) + \frac{1}{2} \Delta p_t.
\end{equation}
We verify the stationarity by setting $p_t(\vx) = p(\vx)$. Using the identity $(\nabla \log p) p = \nabla p$, the drift term becomes $-\frac{1}{2} \nabla \cdot (\nabla p) = -\frac{1}{2} \Delta p$.
This exactly cancels the diffusion term $\frac{1}{2} \Delta p$, yielding $\frac{\partial p_t}{\partial t} = 0$. Thus, the dynamics inherently drive any distribution towards $p(\vx)$, correcting deviations accumulated from prior steps.
\end{proof}

\begin{proposition}[Equivalence of Stochastic Term.]
  Our proposed stochastic sampling step, which acts on the noise prediction component, acts as a valid discretization of a reverse-time SDE by introducing an explicit diffusion term to the standard Rectified Flow ODE.   
\end{proposition}

\begin{proof}
 Recall that the standard deterministic (ODE) update in Rectified Flow is given by linear interpolation: \begin{equation} \vx_{t_\text{next}} = (1-t_\text{next})\hat{\vx}_{0|t} + t_\text{next}\hat{\vx}_{1|t}. \end{equation} Our method introduces stochasticity by perturbing the target noise prediction $\hat{\vx}{1|t}$. Specifically, we replace $\hat{\vx}_{1|t}$ with $\hat{\vx}_{1|t}^{\text{stoch}} = \sqrt{1-\gamma} \hat{\vx}_{1|t} + \sqrt{\gamma} \boldsymbol{\epsilon}$, where $\boldsymbol{\epsilon} \sim \mathcal{N}(0, \mathbf{I})$ and $\gamma$ is a scheduling parameter. Substituting this into the update rule yields: \begin{align} \vx_{t_\text{next}}^{\text{SDE}} &= (1-t_\text{next})\hat{\vx}_{0|t} + t_\text{next} \left( \sqrt{1-\gamma} \hat{\vx}_{1|t} + \sqrt{\gamma} \boldsymbol{\epsilon} \right) \\ &= \underbrace{\left[ (1-t_\text{next})\hat{\vx}_{0|t} + t_\text{next}\sqrt{1-\gamma} \hat{\vx}_{1|t} \right]}_{\text{Effective Drift (Deterministic)}} + \underbrace{\left[ t_\text{next}\sqrt{\gamma} \boldsymbol{\epsilon} \right]}_{\text{Effective Diffusion (Stochastic)}}. \end{align} 
 The resulting update equation takes the form of a standard Euler-Maruyama discretization of an SDE ($d\vx = \mathbf{f}(\vx,t)dt + g(t)d\mathbf{w}$). The first term represents the drift (the intended restoration path), while the second term represents the diffusion ($g(t)d\mathbf{w}$), with the noise magnitude scaled by $t_\text{next}\sqrt{\gamma}$. This explicitly proves that our method injects the necessary stochasticity to correct out-of-distribution (OOD) errors during sampling.
\end{proof}

\paragraph{Derivation of closed-form guidance.}
To enforce the identity constraint, we minimize the loss $\mathcal{L} = \|\mathcal{M} \odot (\vx - \hat{\vx}_{0|t})\|^2$ with respect to the noisy latent $\vx_t$. Applying the chain rule yields $\nabla_{\vx_t}\mathcal{L} = (\frac{\partial \hat{\vx}_{0|t}}{\partial \vx_t})^\top \nabla_{\hat{\vx}_{0|t}}\mathcal{L}$. Calculating the exact Jacobian $\frac{\partial \hat{\vx}_{0|t}}{\partial \vx_t}$ requires computationally expensive backpropagation through the diffusion backbone. To achieve an efficient closed-form solution, we follow standard Diffusion Posterior Sampling practice and approximate this Jacobian as a scalar identity matrix (absorbing scaling factors into the step size $\lambda(t)$). Consequently, the gradient simplifies directly to the masked residual $\nabla_{\vx_t}\mathcal{L} \propto - \mathcal{M} \odot (\vx - \hat{\vx}_{0|t})$, enabling fast, derivative-free guidance updates.

\section{Background}

\subsection{3D Gaussian Splatting}
3D Gaussian Splatting (3D-GS)~\cite{kerbl3Dgaussians} is a photorealistic 3D scene representation and real-time rendering technique~\cite{cui2025streetsurfgs}. Instead of using traditional polygons or volumetric grids, 3D-GS models a scene as a collection of millions of explicit, anisotropic 3D Gaussians. 
Each Gaussian is defined by several key properties: its 3D position (mean), shape (a 3D covariance matrix, allowing it to be a sphere, needle, or flat disk), color (often represented by Spherical Harmonics to capture view-dependent effects), and opacity (alpha). The scene is created by optimizing these properties, typically starting from a sparse point cloud generated by Structure-from-Motion (SfM). During this optimization, a process of adaptive density control dynamically adds (clones) or removes (prunes) Gaussians to efficiently reconstruct fine details. To render a new view, these 3D Gaussians are projected onto the 2D image plane, sorted by depth, and alpha-blended back-to-front in a highly efficient rasterization process.

\subsection{4D Gaussian Splatting}
To extend 3D-GS to dynamic scenes, 4D Gaussian Splatting (4D-GS)~\cite{Wu_2024_CVPR} techniques model how Gaussians move and change over time. Instead of storing separate 3D-GS models for each frame, a holistic 4D representation is learned. A common strategy is to define a set of canonical 3D Gaussians and then predict their deformation at any given timestamp. To efficiently encode this 4D space-time information, methods often employ a decomposed neural voxel grid, drawing inspiration from {HexPlane}~\cite{cao2023hexplane}. This approach factorizes the 4D space ($x, y, z, t$) into several lower-dimensional planes (e.g., $xy$, $xz$, $yt$). To find a Gaussian's deformation, its 4D coordinates are used to query features from these planes. The aggregated features are then passed through a lightweight MLP to predict the transformation (such as translation or rotation), allowing the scene to be reconstructed at novel times.

\subsection{Video Diffusion Transformer Backbone}
Our framework leverages the Wan~\cite{wan2025} architecture, a state-of-the-art text-to-video model built upon the Diffusion Transformer~\cite{ma2024sit, Peebles2022DiT, cui2025optimizing} (DiT) paradigm. This architecture consists of three core components: 
1) A spatio-temporal VAE that compresses input videos from pixel space into a compact latent space; 
2) A robust text encoder (e.g., umT5), selected for its multilingual capabilities and convergence properties, to encode text prompts; 
3) The central Diffusion Transformer, which processes sequences of video latent tokens. 
Within the Transformer blocks, text conditions are injected via cross-attention to ensure semantic fidelity. Temporal information is embedded using a shared MLP that predicts modulation parameters for each block, a design that efficiently enhances performance with minimal parameter overhead.

The model is trained using the Flow Matching framework, specifically Rectified Flows (RF), which provides a stable and theoretically grounded generative process. RF models the transition from pure noise $\vx_0$ to the real data latent $\vx_1$ as a linear interpolation (Ordinary Differential Equation). For a time step $t \in [0, 1]$, the training input $\vx_t$ is defined as:
\begin{equation}
    \vx_t = t \cdot \vx_1 + (1 - t) \cdot \vx_0.
\end{equation}
The model is trained to predict the velocity field $\mathbf{v}_t$ of this trajectory, where the ground truth velocity is simply $\mathbf{v}_t = \vx_1 - \vx_0$. The training objective minimizes the mean squared error (MSE) between the predicted and ground truth velocity. Training follows a multi-stage curriculum, progressing from low-resolution images to high-resolution joint image-video training.

\section{More Implementation Details}

\subsection{Preserved Area Segmentation}
To ensure identity preservation, we define a preserved area mask $\mathcal{M}$. We utilize Grounded-DINO-SAM2 to segment the human region, denoted as $\mathcal{M}_{\text{human}}$, and the garment region, $\mathcal{M}_{\text{garment}}$. The final preserved area is obtained by excluding the garment region from the human mask:
\begin{equation}
    \mathcal{M} = \mathcal{M}_{\text{human}} \setminus \mathcal{M}_{\text{garment}}.
\end{equation}
To align with the latent space of the Video VAE, we downsample the binary mask $\mathcal{M}$ to match the latent dimensions (specifically, downsampling by a factor of 8 spatially and 4 temporally).

\subsection{Residual Field Configuration}
For the non-rigid motion modeling, we employ a multi-resolution HexPlane module. The base resolution $R(i, j)$ is set to 64 and is progressively upsampled by a factor of 2. The Gaussian deformation decoder is implemented as a lightweight MLP using zero-initialization for the final layer weights, ensuring the deformation field starts as an identity mapping.

\subsection{Personalized Video Diffusion}

\paragraph{Control DiT via Channel-wise Concatenation.}
We inject dense spatiotemporal conditions (e.g., the control video) directly into the main branch via latent space augmentation. Unlike adapter-based methods that operate on intermediate features, we concatenate the encoded control latents $\mathbf{y}$ with the noisy video latents $\vx$ along the channel dimension prior to the patch embedding layer. Formally, the input to the DiT becomes $[\vx; \mathbf{y}] \in \mathbb{R}^{B \times (C_x + C_y) \times F \times H \times W}$. This strategy ensures that every spatial patch processed by the Transformer is explicitly conditioned on the corresponding local structural information.

\paragraph{Reference Image Fusion.}
To achieve appearance transfer, we treat the reference image as a visual prompt. The reference image is encoded into latents and passed through a projection layer to match the embedding dimension of the DiT. These projected features are flattened and concatenated with the video tokens along the sequence dimension, effectively serving as a ``visual prefix." By integrating the reference signal into the input sequence, the DiT utilizes its global self-attention mechanism to attend to reference appearance details across all generated frames. These prefix tokens are masked out during the final video reconstruction.

\paragraph{Training Details.}
To facilitate the personalized video generation, we implement the proposed Wan-Control framework based on the \texttt{DiffSynth} library. The model is fine-tuned on a curated subset of the TikTok dataset (available via HuggingFace), comprising approximately 20,000 video clips. For high-fidelity motion guidance, we pre-process all video frames using \texttt{DWPose} to extract dense human pose annotations. 
The training process is conducted on a cluster of $8 \times$ NVIDIA RTX A6000 GPUs for approximately 15,000 iterations. We utilize a constant learning rate with a batch size optimized for the GPU memory. For further architectural details and hyper-parameter configurations.
We also find some good open-sourced alternative, such as \href{https://huggingface.co/alibaba-pai/Wan2.1-Fun-V1.1-1.3B-Control}{https://huggingface.co/alibaba-pai/Wan2.1-Fun-V1.1-1.3B-Control}, and \href{https://huggingface.co/alibaba-pai/Wan2.2-Fun-A14B-Control}{https://huggingface.co/alibaba-pai/Wan2.2-Fun-A14B-Control}, as our \textbf{video backbone}.

\subsection{Baseline Implementation}
We mainly classify our baselines in \cref{tab:difference}, selecting several representative methods with official 
implementation\footnote{For example, since Human4DiT/CharactorShot is not open-sourced, we choose SV4D as a representative method of reconstructing from MV video.} for comparison.

\paragraph{Disco4D~\cite{pang2025disco4d}.}
Due to the unavailability of key components in the official repository, we re-implemented the core algorithm within our own framework. 
Following the supervision strategy of DreamGaussian4D~\cite{ren2023dreamgaussian4d}, this method combines Mean Squared Error (MSE) loss from a single-view driving video with Score Distillation Sampling (SDS) guidance from Zero-123~\cite{liu2023zero1to3}. 
To ensure a fair comparison, we generated the required driving video using our personalized Wan-based model, conditioned on the front-view skeleton rendering and the reference image. We adopted the SDS implementation directly from the DreamGaussian4D repository.

\paragraph{SV4D 2.0~\cite{yao2024sv4d2}.}
SV4D is a multi-view video diffusion model fine-tuned on Stable Video Diffusion (SVD) using a large-scale 4D dataset filtered from Objaverse. It takes a single-view video as input and outputs synchronized multi-view videos. 
We utilized the same driving video generated for Disco4D as the input. 
However, we observe a severe identity shift between the output and input videos. We attribute this to the domain gap, as SV4D is trained primarily on synthetic Objaverse objects rather than realistic human captures.

\paragraph{PERSONA~\cite{sim2025persona}.}
We utilize the official implementation of PERSONA. 
This method employs MimicMotion~\cite{zhang2025mimicmotion}, a pose-driven video diffusion model, to generate synthetic video data which is then used to optimize a canonical 3D Gaussian field and a pose-dependent deformation field. 
The pipeline relies on an extensive set of off-the-shelf components, including Sapiens~\cite{khirodkar2024sapiens}, DECA, and ResShift.
Despite incorporating various regularization terms, such as geometry weighted optimization and multiple monocular normal/depth priors, we find that the method struggles to preserve the fine-grained identity of the subject during complex motions.

\paragraph{LHM~\cite{qiu2025LHM}.}
We use the official implementation of LHM as a representative kinematics-based baseline. 
LHM effectively reconstructs a human from a single-view image with high-fidelity identity and efficient inference speed. 
However, as it relies purely on kinematics-based deformation to animate the 3D Gaussians, it fundamentally lacks the ability to model non-rigid dynamics such as clothing deformation. 
(Note: Our method builds upon this kinematics-based representation, using it as a starting point to learn residual non-rigid motions via video diffusion priors.)

\begin{table}[]
    \centering
    \resizebox{\linewidth}{!}{
    \begin{tabular}{lcccccccc}
    
    \toprule 
        Method & Training Objective & Single-view  Image Input & Skeleton-controllable & Identity Preservation & Non-rigid Motion  & High-quality Rendering\\
        \midrule 
        Disco4D~\cite{pang2025disco4d} & MSE+SDS & $\checkmark$ & $\times$ &  $\checkmark$ & $\times$ & $\times$  \\
        AKD~\cite{li2025akd} & SDS & $\checkmark$ & $\times$ & $\checkmark$ & $\times$ & $\times$ \\
        PhysAvatar~\cite{PhysAavatar24} & MSE & $\times$ & $\checkmark$ & $\checkmark$ & $\checkmark$ & $\checkmark$ \\
        SV4D/SV4D 2.0~\cite{xie2024sv4d, yao2024sv4d2} & MSE  & $\checkmark$ & $\times$ & $\times$ & $\checkmark$ & $\times$ \\
        CharacterShot~\cite{gao2025charactershot} & MSE & $\checkmark$ & $\checkmark$ & $\times$ & $\checkmark$ & $\times$ \\
        Human4DiT~\cite{shao2024human4dit}  & MSE & $\checkmark$ & $\checkmark$ & $\times$ & $\checkmark$ & $\times$ \\
        PERSONA~\cite{sim2025persona} & MSE & $\checkmark$ & $\checkmark$ & $\times$ & $\checkmark$ & $\times$ \\
        LHM~\cite{qiu2025LHM} & - & $\checkmark$ & $\checkmark$ & $\checkmark$ & $\times$ & $\times$  \\
        Ours & MSE & $\checkmark$ & $\checkmark$ & $\checkmark$ & $\checkmark$ & $\checkmark$ \\
    \bottomrule 
    \end{tabular}
    }
    \caption{Difference among the other human (character) animation methods.
    ``-'' means there is no optimization process in the animation.}
    \label{tab:difference}
\end{table}

\subsection{Details of Competitive Sampling Methods}
We compare our approach against several representative methods capable of transforming low-quality source inputs into high-quality targets using off-the-shelf diffusion priors.
Since some of algorithms are designed from DDPM, we implement them in the context of flow matching with their core ideas.

\paragraph{Vanilla SDEdit~\cite{meng2022sdedit}.}
SDEdit serves as the foundational baseline for image and video restoration. The method follows a strictly stochastic process: it first perturbs the source input $\vx_\text{src}$ by adding Gaussian noise to reach an intermediate time step $t_{0} \in (0, 1)$. This forward diffusion process effectively destroys high-frequency artifacts. Subsequently, the standard reverse ODE/SDE sampling is applied from $t_{0}$ to $t=0$ to generate the restored output. While effective for minor denoising, it often faces a trade-off between preserving identity (low $t_{0}$) and removing significant artifacts (high $t_{0}$).

\paragraph{MCS~\cite{wang2024vistadream}.}
Multiview Consistency Sampling (MCS) was originally proposed to balance fidelity and generation quality in 3D scene generation. The authors observe that while higher noise injection improves realism, it degrades the structural fidelity to the input. To mitigate this, MCS modifies the posterior mean during sampling to explicitly include signal from the input image. 
In our implementation, we adapt this to the Flow Matching framework. At each sampling step, we modify the predicted posterior mean $\hat \vx_{0|t}$ to incorporate a weighted component of the source input $\vx_\text{src}$.
This bias term forces the generation trajectory to remain structurally close to the input video, ensuring that the ``hallucinated" details align with the original identity.

\paragraph{HFS-SDEdit~\cite{elevating3d}.}
HFS-SDEdit aims to preserve structural details by explicitly fusing frequency components in the latent space. It operates on the hypothesis that the structural identity resides in high-frequency signals. During the reverse sampling process, the method replaces the high-frequency component of the current denoised latent $\vx_t$ with that of the noisy source input. The update rule is defined as:
\begin{equation}
    \vx'_t = \text{LPF}(\vx_t) + \text{HPF}(\vx_{\text{src},t}),
\end{equation}
where $\text{LPF}$ and $\text{HPF}$ denote Gaussian low-pass and high-pass filters, respectively, and $\vx_{\text{src},t}$ is the noised version of input data corresponding to time $t$. This operation forces the solver to generate realistic low-frequency content (lighting, materials) while rigorously adhering to the edges and boundaries of the source input.

\paragraph{NC-SDEdit~\cite{yang2024noise}.}
We adapt the Noise Calibration (NC) strategy to our Flow Matching framework. While the original implementation calibrates the noise estimate $\boldsymbol{\epsilon}$, our adaptation operates directly on the estimated clean data (posterior) to ensure structural consistency.
In each sampling step $t$, we first solve the flow equation to estimate the clean target $\hat{\vx}_{0|t}$ from the current noisy latent $\vx_t$ and predicted velocity $\vv_t$.
We then calibrate this posterior by replacing its high-frequency components with those of the source reference $\vx_\text{src}$:
\begin{equation}
    \hat{\vx}'_{0|t} = \hat{\vx}_{0|t} - \text{HPF}(\hat{\vx}_{0|t}) + \text{HPF}(\vx_\text{src}),
\end{equation}
where $\text{HPF}(\cdot)$ extracts high-frequency details via Fourier transform. 
Finally, the solver (e.g., Euler step) computes the latent for the next timestep $\vx_{t_\text{next}}$ using this calibrated target $\hat{\vx}'_{0|t}$. This approach enforces strict structural alignment with the input video throughout the generation trajectory while allowing the low-frequency content to be refined by the diffusion prior.

\paragraph{FlowEdit~\cite{kulikov2024flowedit}.}
FlowEdit constructs a mapping between source and target distributions by leveraging the reversibility of ODEs. It defines the editing direction based on the difference between a source velocity (conditioned on a source prompt) and a target velocity (conditioned on a target prompt). 
In our experiments, we utilize a negative prompt (describing low-quality attributes) to match the source distribution and a positive prompt for the target. However, we find that because FlowEdit relies on the model's semantic understanding of the prompt to model the degradation, it often fails to correct the severe, non-semantic out-of-distribution (OOD) artifacts present in the coarse 3D renderings, as these specific artifacts are not easily described by text.

\section{Ablation Studies (Extended)}

\subsection{Quantitative Ablation}
Table~\ref{tab:r1_supp} provides a comprehensive quantitative evaluation of each key component within our framework, including the stochastic sampling mechanism, the self-guidance strategy, and the personalized video diffusion module. According to the results, our full model configuration achieves the optimal balance between visual fidelity and identity preservation. Specifically, while the exclusion of self-guidance leads to a marginal improvement in the Frechet Inception Distance, it incurs a substantial degradation in identity consistency, as reflected by the significant drop in the CLIP-Identity score. This observation validates that self-guidance is indispensable for maintaining the subject's unique features throughout the generation process. Furthermore, the integration of stochastic sampling and personalized diffusion proves essential for temporal coherence and motion realism, with the full model yielding the lowest Frechet Video Distance. Although individual modules may favor specific metrics, the synergistic effect of all components ensures that the model produces high-quality videos without compromising the structural or stylistic integrity of the personalized target.

\begin{table}[h]\centering\caption{Quantitative ablation study of the proposed components. The results demonstrate that the full model configuration achieves the most robust performance across all evaluation metrics.}\label{tab:r1_supp}\begin{tabular}{lccccc}\toprule
Metrics & Coarse Model & w/o Stochastic & w/o Self-Guidance & w/o Personalized & \textbf{Full Model} \\
\midrule
FID $\downarrow$    & 199.1 & 187.4 & \textbf{104.1} & 125.3 & \underline{105.3}\\
FVD $\downarrow$    & 367.0 & 349.7 & \underline{298.8} & 301.4 & \textbf{295.2}\\
CLIP-Identity $\uparrow$   & \textbf{0.8847} & 0.8804 & 0.8220 & 0.8645 & \underline{0.8838} \\
\bottomrule\end{tabular}\end{table}

\subsection{More Results of Self-guided Stochastic Sampling}
Due to space constraints in the main manuscript, we provide additional qualitative comparisons to validate the efficacy of our core technical contribution: \textbf{self-guided stochastic sampling}. 
We evaluate our approach against two distinct baselines: 
1) \textbf{Direct Generation}, which uses the pretrained video model directly (with reference image and 2D skeleton sequence); and 
2) \textbf{Standard ODE-based Restoration}, where we employ MCS~\cite{wang2024vistadream} as a representative deterministic sampling method. 
Dynamic visualizations of these comparisons can be found in the Supplementary Video (\texttt{00:40} - \texttt{01:32}).

As illustrated in \cref{fig:supp-abl1}, the challenges of this task are evident. The initial \textbf{mesh-rigged animation} (Input) exhibits significant artifacts, including unnatural garment dynamics and blurred edges, consistent with the limitations discussed in the main paper (e.g., Fig. 1). 
\textbf{Direct Generation}, while achieving high realism, suffers from severe identity loss and hallucinations; notably, the model generates extraneous accessories such as a bag (Row 2) or a watch (Row 3), rendering it unsuitable for faithful reconstruction. 
Furthermore, standard \textbf{ODE-based sampling} (MCS) fails to effectively correct the out-of-distribution nature of the coarse rendering, resulting in over-smoothed textures and persistent blurring along garment boundaries. 
In contrast, our \textbf{self-guided stochastic sampling} effectively bridges the gap between realism and fidelity. It restores photorealistic details and valid non-rigid dynamics while strictly preserving the original human identity.

\subsection{Sensitivity Analysis of Initial Noise Strength}

The initial noise strength, denoted as $t_0$, serves as the critical hyperparameter in our self-guided stochastic sampling strategy. It governs the trade-off between the restoration capability and the fidelity to the initial coarse rendering. As illustrated in \cref{fig:supp-abl-str}, we conduct a comprehensive sensitivity analysis by varying $t_0$ across the range $[0.2, 0.8]$. 
At lower noise levels ($t_0 \in \{0.2, 0.4\}$), the sampling trajectory is too short to effectively correct the Out-of-Distribution (OOD) artifacts, resulting in outputs that retain the degradation of the source mesh-rigged animation. 
Conversely, at higher noise levels ($t_0 \in \{0.6, 0.8\}$), our method demonstrates significant robustness. Unlike standard restoration methods where high noise often leads to identity loss, our self-guidance mechanism ensures that the subject's identity remains remarkably stable even at $t_0 = 0.8$. 
Ultimately, we empirically select $t_0 = 0.6$ as the default setting, as it strikes an optimal balance between generation quality, identity preservation, and sampling efficiency.
\begin{figure}[t!]
    \centering
    \includegraphics[width=0.8\linewidth]{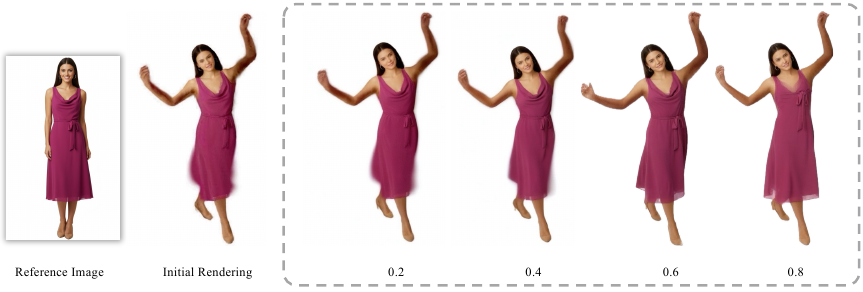}
    \caption{\textbf{Sensitivity analysis of the initial noise strength $t_0$.} 
    We visualize restoration results across varying noise strengths. 
    Low noise levels ($t_0=0.2, 0.4$) fail to deviate sufficiently from the source, leaving artifacts from the coarse mesh rendering intact. 
    Higher noise levels ($t_0=0.6, 0.8$) effectively hallucinate plausible details and correct non-rigid dynamics. 
    Notably, thanks to our self-guidance mechanism, the identity is preserved even at high noise strengths ($t_0=0.8$), overcoming the traditional quality-fidelity trade-off.}
\label{fig:supp-abl-str}  
\end{figure}

\subsection{More Ablations in 4D Optimization}

\paragraph{Adaptive densification.}
Adaptive densification and pruning are fundamental mechanisms in 3D Gaussian Splatting for capturing high-frequency details. We incorporate these strategies into our photorealistic 4D reconstruction pipeline. 
As demonstrated in \cref{fig:supp-abl-densify}, relying solely on the deformation of the canonical geometry is insufficient to model complex texture dynamics (e.g., shifting wrinkles). Without densification, the model fails to allocate sufficient primitives to these dynamic regions, causing the clothing textures to appear significantly blurred.

\paragraph{Mask loss regularization.}
Prior works, such as PERSONA, have established that geometric constraints are critical for fidelity. We validate this by ablating the mask loss during our 4D optimization. 
This regularization is particularly important in conjunction with our densification strategy. As shown in \cref{fig:supp-abl-mask}, without the mask loss to constrain the generation of new primitives, ``floaters" emerge in free space, and the boundary definition of the subject degrades compared to our full setting.

\begin{figure}[t]
    \centering
    \begin{subfigure}[b]{0.4\linewidth}
        \centering
        \includegraphics[width=\linewidth]{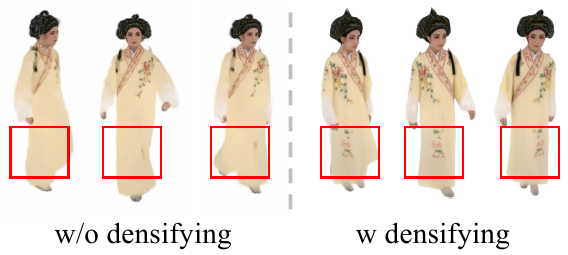}
        \caption{\textbf{Effect of Densification.}}
        \label{fig:supp-abl-densify}
    \end{subfigure}
    \hfill
    \begin{subfigure}[b]{0.52\linewidth}
        \centering
        \includegraphics[width=\linewidth]{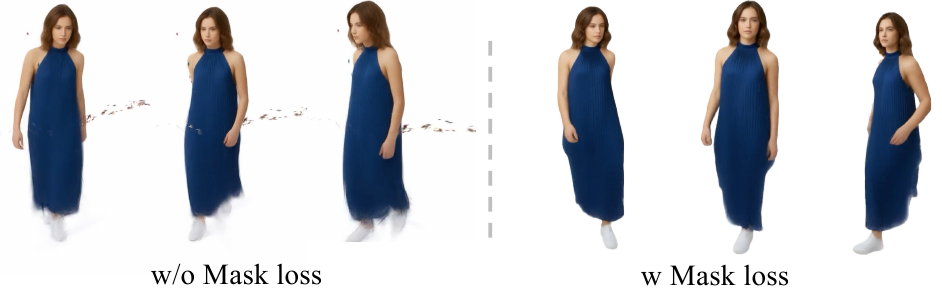}
        \caption{\textbf{Effect of Mask Loss.}}
        \label{fig:supp-abl-mask}
    \end{subfigure}
    
    \vspace{2mm}
    \begin{subfigure}[b]{0.48\linewidth} 
        \centering
        \includegraphics[width=\linewidth]{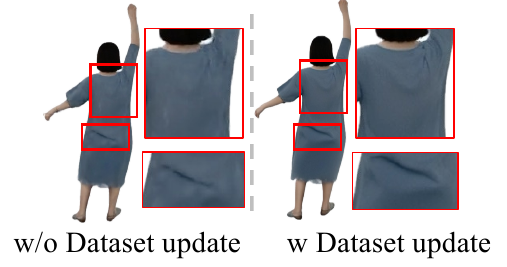}
        \caption{\textbf{Effect of Dataset Update.}}
        \label{fig:supp-abl-du}
    \end{subfigure}

    \vspace{-1mm}
    \caption{\textbf{Ablations on optimization strategies.} 
    (a) {Adaptive densification} is crucial for capturing high-frequency texture dynamics. 
    (b) {Mask loss regularization} is essential to constrain the geometry.
    (c) {Dataset update} mitigates over-smoothing caused by inconsistent supervision, allowing the model to converge on sharp, clear details.}
    \label{fig:abl-optim-combined}
    \vspace{-3mm}
\end{figure}
\begin{figure}
    \centering
    \includegraphics[width=0.73\linewidth]{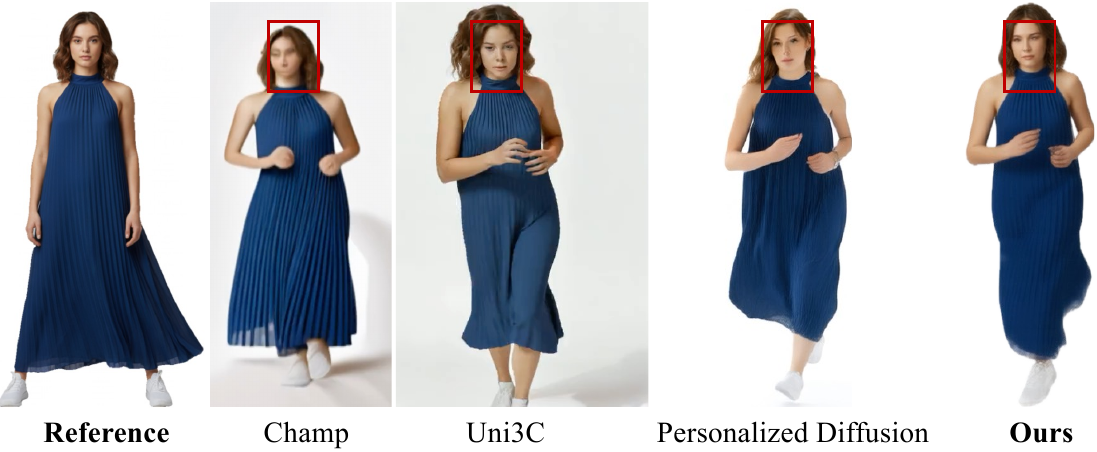}
    \caption{Comparison with image-based animation.}
    \label{fig:comp-image}
\end{figure}

\begin{figure}[h!] \centering
\includegraphics[width=\textwidth]{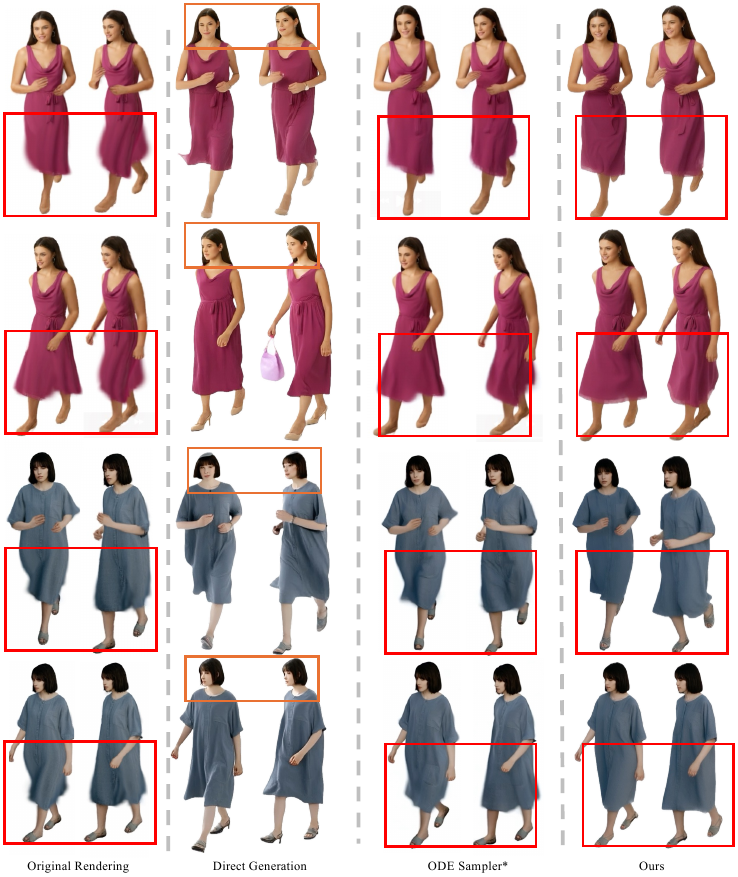}
\caption{\textbf{Visual comparison of sampling strategies.} Cases I: Two girls are walking (Row2/4) and running (Row1/3). The {Mesh-Rigged Animation} (Input) exhibits unrealistic artifacts, such as unnatural cloth dynamics and blurry edges. {Direct Generation} suffers from severe identity shift, introducing hallucinations like a bag (Row 2) or a watch (Row 3). {ODE Sampling} (represented by MCS~\cite{wang2024vistadream}) fails to recover high-frequency details, leaving garment edges blurry due to the OOD nature of the input. In contrast, {Ours} successfully restores high-fidelity details and realistic motion while maintaining strict identity consistency.    } \label{fig:supp-abl1}
\end{figure}

\begin{figure}[h!] \centering
\includegraphics[width=\textwidth]{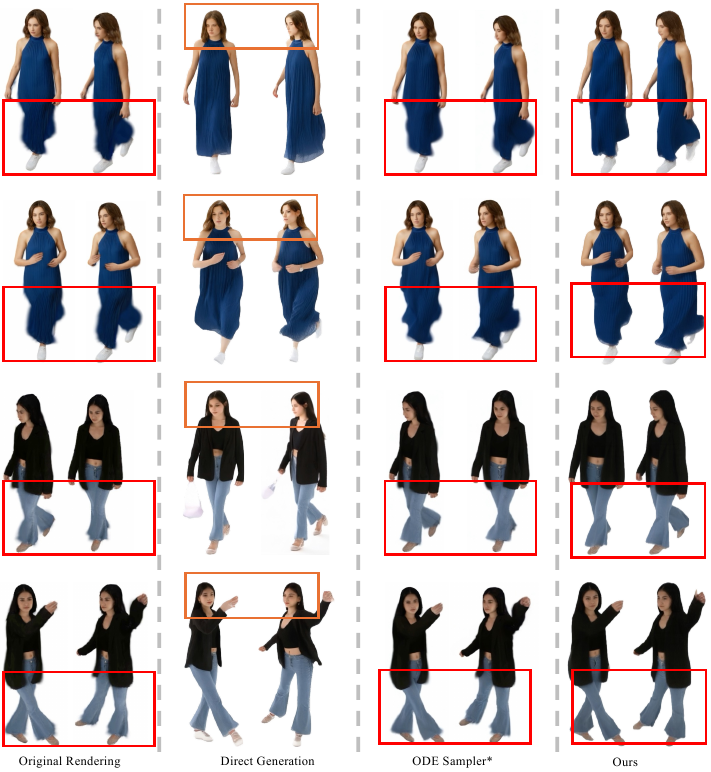}
    \caption{Visual comparison of sampling strategies. Case II: Two girls are walking (Row1/3), running (Row2), dancing (Row4).} \label{fig:supp-abl2}
\end{figure}

\begin{figure}[h!] \centering
    \includegraphics[width=\textwidth]{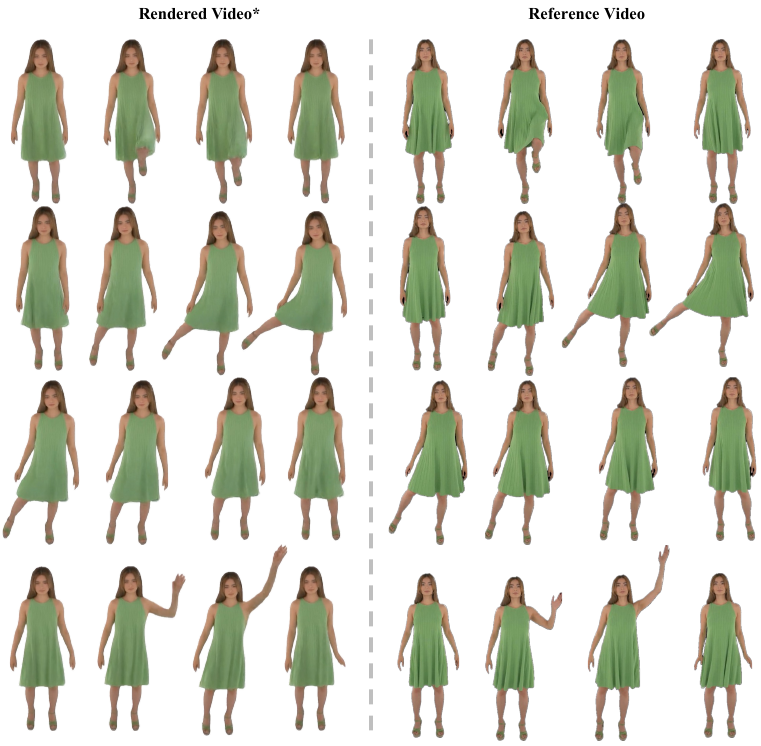}
\caption{\textbf{Qualitative evaluation on the ActorsHQ~\cite{isik2023humanrf} dataset (I).} 
We show single person with difference motions.
The asterisk (*) denotes renderings at a specific viewpoint (elevation $10^\circ$, azimuth $0^\circ$). 
Note that slight spatial misalignments between the generation and ground truth are due to inherent errors in the SMPL estimation derived from the source video. 
Despite relying on a single-view input, our method faithfully preserves human identity and captures complex non-rigid deformations (e.g., dress dynamics), even during extreme poses such as high leg raises.
}
\label{fig:append_com1}\end{figure}
\begin{figure}[h!] \centering
    \includegraphics[width=0.98\textwidth]{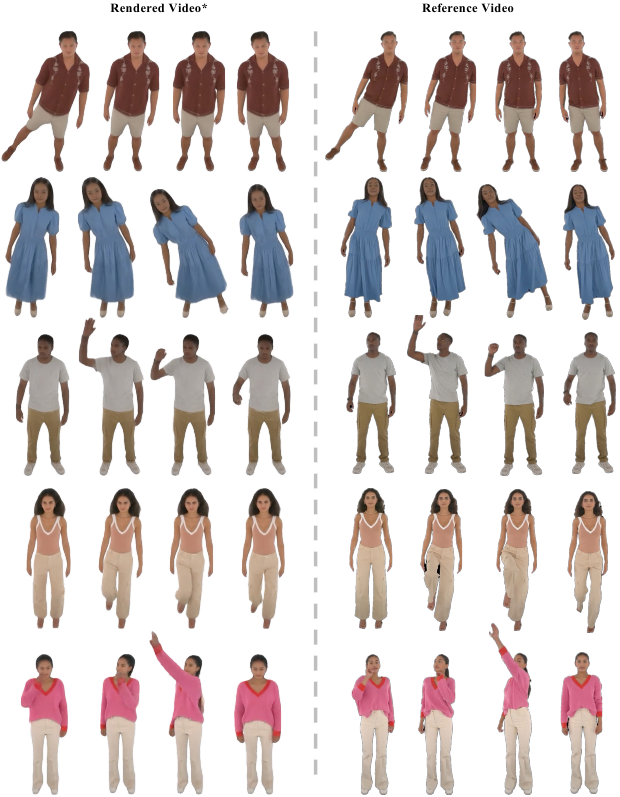}
    \caption{\textbf{Human reconstruction results in ActorsHQ~\cite{isik2023humanrf} dataset (II).}
    We show different person with diverse motions.} \label{fig:append_com2}
\end{figure}

\begin{figure}[h!] \centering
    \includegraphics[width=0.9\textwidth]{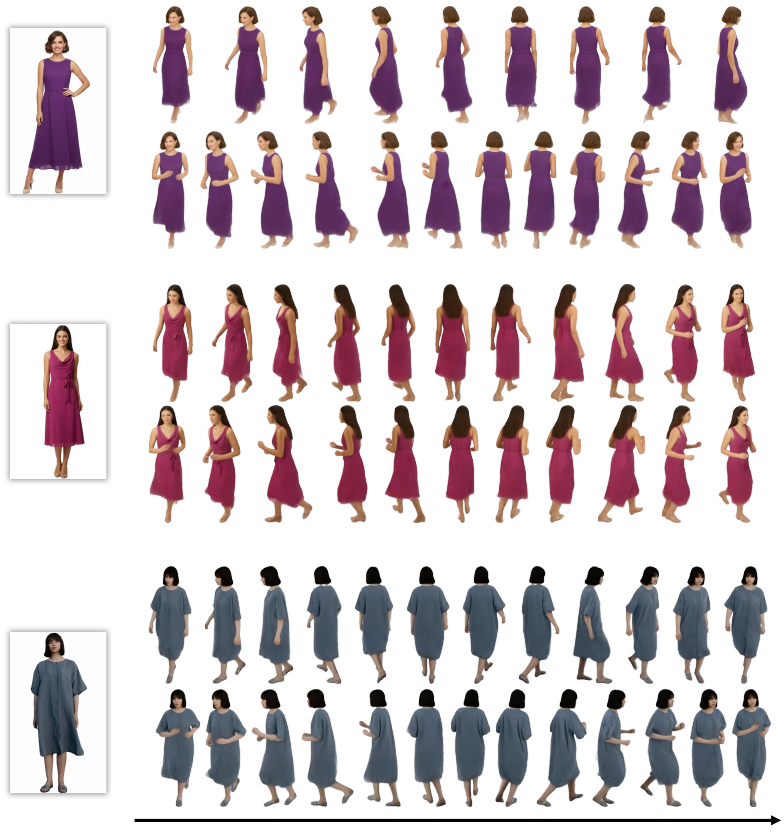}
    \caption{\textbf{Additional human animation results.} 
    We visualize diverse subjects performing various motions, rendered with dynamic 360-degree camera trajectories.} 
\label{fig:append_g1}
\end{figure}

\paragraph{Iterative dataset update.}
We compare our iterative dataset update strategy against a standard single-stage optimization.
As illustrated in \cref{fig:supp-abl-du}, single-stage optimization tends to result in over-smoothed textures, effectively ``averaging out" high-frequency details due to inherent view and temporal inconsistencies in the initial supervision.
In contrast, employing the dataset update mechanism allows the optimization to reject inconsistent noise and converge towards high-fidelity results, significantly sharpening fine-grained features such as dress wrinkles.

\section{Results (Extended)}
\subsection{Training Efficiency}
From a single image, we adopt LHM to obtain the canonical 3D Gaussians, and generate the basic mesh-rigged animations with prepared SMPLX mesh sequences within 1 minute.
During re-rendering and 4D optimization, we have $30k$ optimization iterations in total, and update our generated pseudo-ground truth per-$5k$ iterations.
Each video re-rerendering (sampling) step takes about $67$s in average, and we simultaneously update each trajectory.
The overall time cost is about $19$ mins.
In contrast, PERSONA needs more than 6 hours to create an animation (more than 4 hours for complex data preprocessing and long-sequence video generation, and additional $>1$ hour optimization).

\subsection{Discussion with Image-based Animation methods.}
To further evaluate the effectiveness of our framework, we compare our method with state-of-the-art image-driven animation models, including Champ and Uni3C, as well as our backbone, Personalized Diffusion. As summarized in Table~\ref{tab:r2_supp}, while image-based methods such as Uni3C achieve competitive rendering quality in terms of Frechet Video Distance, they struggle to maintain high identity consistency, particularly in challenging side-view perspectives. This is reflected in their lower CLIP-Identity scores compared to our approach. Our method consistently outperforms these baselines by leveraging the robust identity priors of the personalized video diffusion model.

\begin{table}[h]\centering\caption{Quantitative comparison with state-of-the-art image-driven animation methods. Our method achieves a superior balance between motion fidelity and identity preservation.}\label{tab:r2_supp}\begin{tabular}{lcccc}\toprule
Metrics & Champ & Uni3C & Personalized Diffusion & {Ours} \\
\midrule
FID $\downarrow$    & 196.3 & 132.3 & 138.5 & \textbf{112.8}\\
FVD $\downarrow$    & 467.2 & \textbf{284.4} & 330.6 & \underline{289.0} \\
CLIP-Identity $\uparrow$   & 0.7633 & 0.8357 & 0.8001 & \textbf{0.8844}\\
\bottomrule\end{tabular}\end{table}

A key advantage of our framework over video-based animation methods is the ability to \textbf{distill} the pose-controlled video diffusion model into a 4D Gaussian Splatting (4DGS) representation. Traditional video diffusion models require a time-consuming iterative denoising process to generate each sequence. In contrast, once our distillation process is complete, the \textbf{\textit{resulting 4D Gaussian representation allows for high-fidelity, real-time rendering of the personalized character in any viewpoint}} . This shift from generative inference to rasterization-based rendering significantly reduces the computational latency, making our approach highly suitable for interactive applications that require both personalized identity and responsive motion control.



\subsection{Results in ActorsHQ dataset~\cite{isik2023humanrf}}
To further assess the generalization capability of our framework, we evaluate \MethodName on the high-fidelity ActorsHQ dataset~\cite{isik2023humanrf}. As shown in \cref{fig:append_com1} and \cref{fig:append_com2}, our method successfully reconstructs and animates the subject using only a single-view image as input. The results demonstrate that our approach effectively handles challenging articulation scenarios, such as high leg raises, while generating plausible non-rigid dynamics for loose clothing (e.g., skirts). We note that some systematic spatial misalignment between our rendering and the ground truth is observed; this is attributable to inaccuracies in the underlying SMPL parameters estimated from the raw video data, rather than a limitation of the generation pipeline itself. Despite this, the method maintains strong identity preservation and temporal consistency.

\subsection{Additional Qualitative Results}
In \cref{fig:append_g1}, we present renderings using dynamic 360-degree camera trajectories across various subjects and complex motions.
These results demonstrate that our framework generalizes effectively to diverse identities and actions, maintaining high visual fidelity and temporal consistency from all viewing angles.

\subsection{Limitations}
Although our framework achieves high-quality results in 3D human animation, it is subject to the inherent limitations of the underlying representation.
Specifically, as we rely on 4D Gaussian Splatting (4DGS), the reconstruction is not strictly lossless. While the method achieves high quantitative metrics (PSNR $\approx$ 35 dB), the discrete nature of the primitives may still result in minor smoothing of extremely high-frequency texture details compared to the source video.

\end{document}